\documentclass{article} 
\usepackage{amssymb}
\usepackage{amsmath}
\usepackage{amsthm}
\usepackage[a4paper]{geometry}
\usepackage[round]{natbib}
\usepackage[colorlinks=true, linkcolor=blue]{hyperref}

\usepackage{amssymb}
\usepackage{amsmath,amscd}
\usepackage{amsthm} 
\usepackage{cancel}     %
\usepackage{centernot}  %
\usepackage{graphicx}   %
\usepackage{soul}       %
\usepackage{color}      %
\usepackage{enumitem}   %
\usepackage{commath}    %

\DeclareFontFamily{U}{mathx}{\hyphenchar\font45}
\DeclareFontShape{U}{mathx}{m}{n}{
      <5> <6> <7> <8> <9> <10> gen * mathx
      <10.95> mathx10 <12> <14.4> <17.28> <20.74> <24.88> mathx12
      }{}
\DeclareSymbolFont{mathx}{U}{mathx}{m}{n}
\DeclareMathSymbol{\intop}  {1}{mathx}{"B3}

\newcommand\indep{\independent}
\newcommand\independent{\protect\mathpalette{\protect\independenT}{\perp}}
\def\independenT#1#2{\mathrel{\rlap{$#1#2$}\mkern4mu{#1#2}}}

\newcommand{\wh}{\widehat}
\newcommand{\wb}{\overline}

\let\temp\phi
\let\phi\varphi
\let\varphi\temp

\renewcommand{\sec}{\textsection}

\newcommand{\pr}{\mathbb{P}}

\newcommand{\R}{\mathbb{R}}

\newcommand{\E}{\mathbb{E}}

\newcommand{\normalN}{\mathcal{N}}
\newcommand{\given}{\,|\,}  %

\renewcommand{\norm}[1]{\Vert#1\Vert} %

\newcommand{\ip}[1]{\langle#1\rangle}

\newcommand{\eps}{\varepsilon}

\newcommand{\les}{\lesssim}
\newcommand{\ges}{\gtrsim}

\newcommand{\iid}{\overset{\text{iid}}{\sim}}

\DeclareMathOperator{\cov}{cov}
\DeclareMathOperator{\var}{var}

\DeclareMathOperator*{\argmin}{arg\,min}

\DeclareMathOperator{\BernoulliDist}{Ber}

\newtheorem{thm}{Theorem}[section]
\newtheorem{lemma}[thm]{Lemma}
\newtheorem{cor}[thm]{Corollary}
\newtheorem{prop}[thm]{Proposition}

\newtheorem*{thm*}{Theorem}
\newtheorem*{lemma*}{Lemma}
\newtheorem*{cor*}{Corollary}
\newtheorem*{prop*}{Proposition}
\newtheorem*{conjecture*}{Conjecture}

\theoremstyle{definition}

\newtheorem*{defn*}{Definition}
\theoremstyle{definition}

\theoremstyle{definition}
\newtheorem{ex}{Example}
\theoremstyle{remark}
\newtheorem*{ex*}{Example}
\theoremstyle{definition}

\theoremstyle{definition}
\newtheorem*{assm*}{Assumption}
\theoremstyle{remark}

\theoremstyle{remark}
\newtheorem*{remark*}{Remark}

\theoremstyle{definition}
\newtheorem{cond}{Condition}

\usepackage{algorithm}
\usepackage{subcaption} %
\usepackage{booktabs}

\newcommand{\err}{z}
\newcommand{\gr}{\mathsf{G}}
\DeclareMathOperator{\pa}{pa}

\newcommand{\ord}{\prec}

\newcommand{\width}{d}
\newcommand{\layer}{L}
\newcommand{\anc}{A}
\newcommand{\diff}{\omega}
\newcommand{\resvarbd}{\xi}
\newcommand{\npestbd}{\delta}

\DeclareMathOperator{\resvar}{RV}

\newcommand{\RESIT}{\text{RESIT}}
\newcommand{\CAM}{\text{CAM}}
\newcommand{\EqVar}{\text{EqVar}}
\newcommand{\NOTEARS}{\text{NOTEARS}}
\newcommand{\GSGES}{\text{GSGES}}
\newcommand{\PC}{\text{PC}}
\newcommand{\GES}{\text{GES}}
\newcommand{\NPVAR}{\text{NPVAR}}

\DeclareMathOperator{\de}{de}

\title{A polynomial-time algorithm for learning nonparametric causal graphs}

\usepackage{authblk}
\author[]{Ming Gao}
\author[]{Yi Ding}
\author[]{Bryon Aragam}
\affil[]{\emph{University of Chicago}}

\begin{document}

\maketitle

{\let\thefootnote\relax\footnote{Contact: \texttt{minggao@uchicago.edu, dingy@uchicago.edu, bryon@chicagobooth.edu}}}

\begin{abstract}
We establish finite-sample guarantees for a polynomial-time algorithm for learning a nonlinear, nonparametric directed acyclic graphical (DAG) model from data.
The analysis is model-free and does not assume linearity, additivity, independent noise, or faithfulness.
Instead, we impose a condition on the residual variances that is closely related to previous work on linear models with equal variances. 
Compared to an optimal algorithm with oracle knowledge of the variable ordering, the additional cost of the algorithm is linear in the dimension $d$ and the number of samples $n$. 
Finally, we compare the proposed algorithm to existing approaches in a simulation study.
\end{abstract}

\section{Introduction}
\label{sec:intro}

Modern machine learning (ML) methods are driven by complex, high-dimensional, and nonparametric models that can capture highly nonlinear phenomena. These models have proven useful in wide-ranging applications including vision, robotics, medicine, and natural language. At the same time, the complexity of these methods often obscure their decisions and in many cases can lead to wrong decisions by failing to properly account for---among other things---spurious correlations, adversarial vulnerability, and invariances \citep{bottou2015two,scholkopf2019causality,buhlmann2018invariance}. This has led to a growing literature on correcting these problems in ML systems. A particular example of this that has received widespread attention in recent years is the problem of causal inference, which is closely related to these issues. While substantial methodological progress has been made towards embedding complex methods such as deep neural networks and RKHS embeddings into learning causal graphical models \citep{huang2018generalized,mitrovic2018causal,zheng2019learning,yu2019dag,lachapelle2019gradient,zhu2019causal,ng2019masked}, theoretical progress has been slower and typically reserved for particular parametric models such as linear \citep{chen2018causal,wang2018nongauss,ghoshal2017ident,ghoshal2017sem,loh2014causal,geer2013,aragam2015ccdr,aragam2015highdimdag,aragam2019globally} and generalized linear models \citep{park2017,park2018learning}.

In this paper, we study the problem of learning directed acyclic graphs (DAGs) from data in a nonparametric setting. Unlike existing work on this problem, we do not require linearity, additivity, independent noise, or faithfulness.
Our approach is model-free and nonparametric, and uses nonparametric estimators (kernel smoothers, neural networks, splines, etc.) as ``plug-in'' estimators. As such, it is agnostic to the choice of nonparametric estimator chosen. Unlike existing consistency theory in the nonparametric setting \citep{peters2014,hoyer2009,buhlmann2014,rothenhausler2018causal,huang2018generalized,tagasovska2018nonparametric,nowzohour2016}, we provide explicit (nonasymptotic) finite sample complexity bounds and show that the resulting method has polynomial time complexity.
The method we study is closely related to existing algorithms that first construct a variable ordering \citep{ghoshal2017ident,chen2018causal,ghoshal2017sem,park2020identifiability}.
Despite this being a well-studied problem, to the best of our knowledge our analysis is the first to provide explicit, simultaneous statistical and computational guarantees for learning nonparametric DAGs.

\begin{figure}[t]
\centering
\begin{subfigure}[t]{0.5\textwidth}
\includegraphics[width=\textwidth]{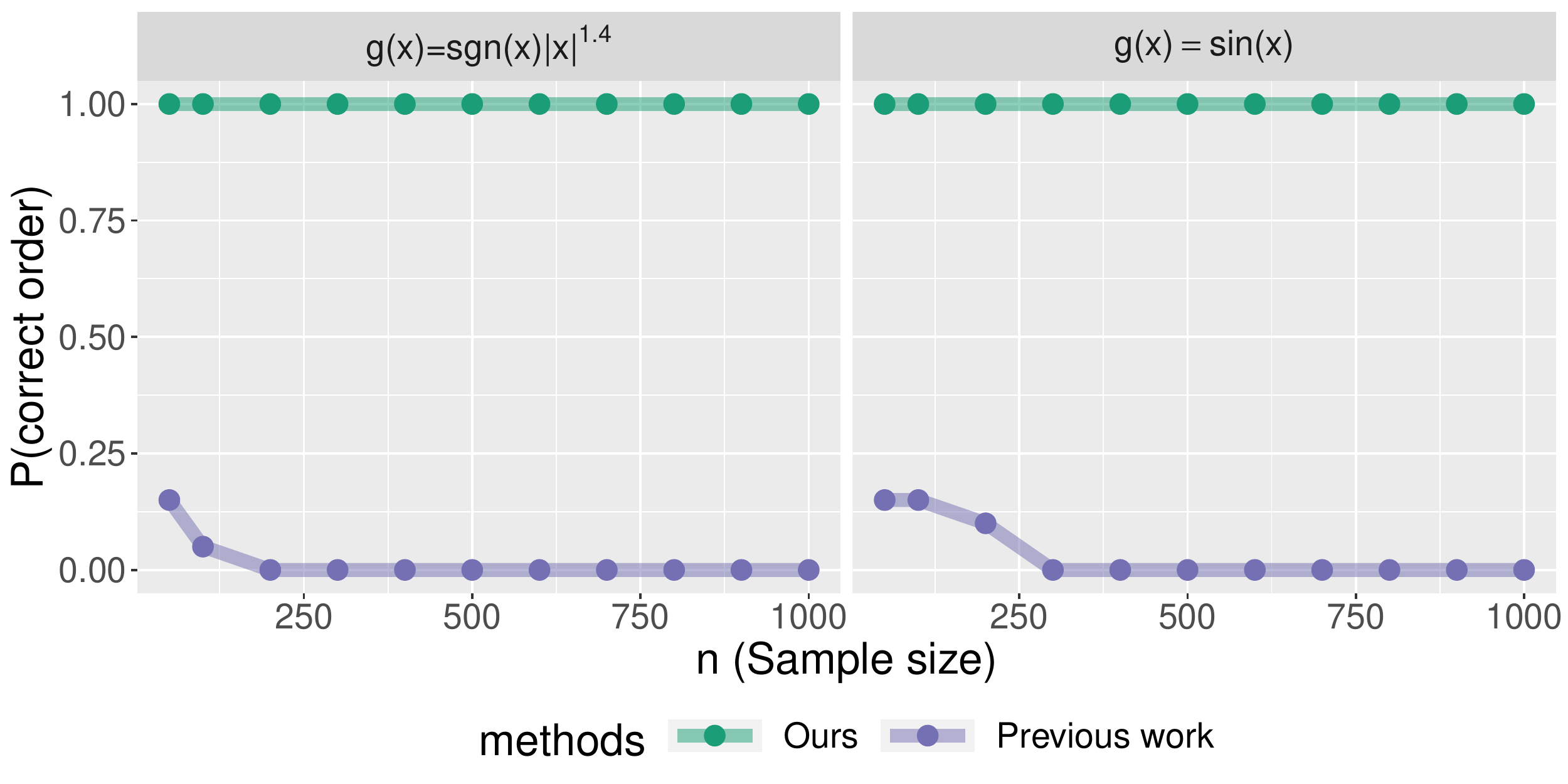}
\caption{}
\label{fig:intro:camfail}
\end{subfigure}
\hspace{1em}
\begin{subfigure}[t]{0.43\textwidth}
\includegraphics[width=\textwidth]{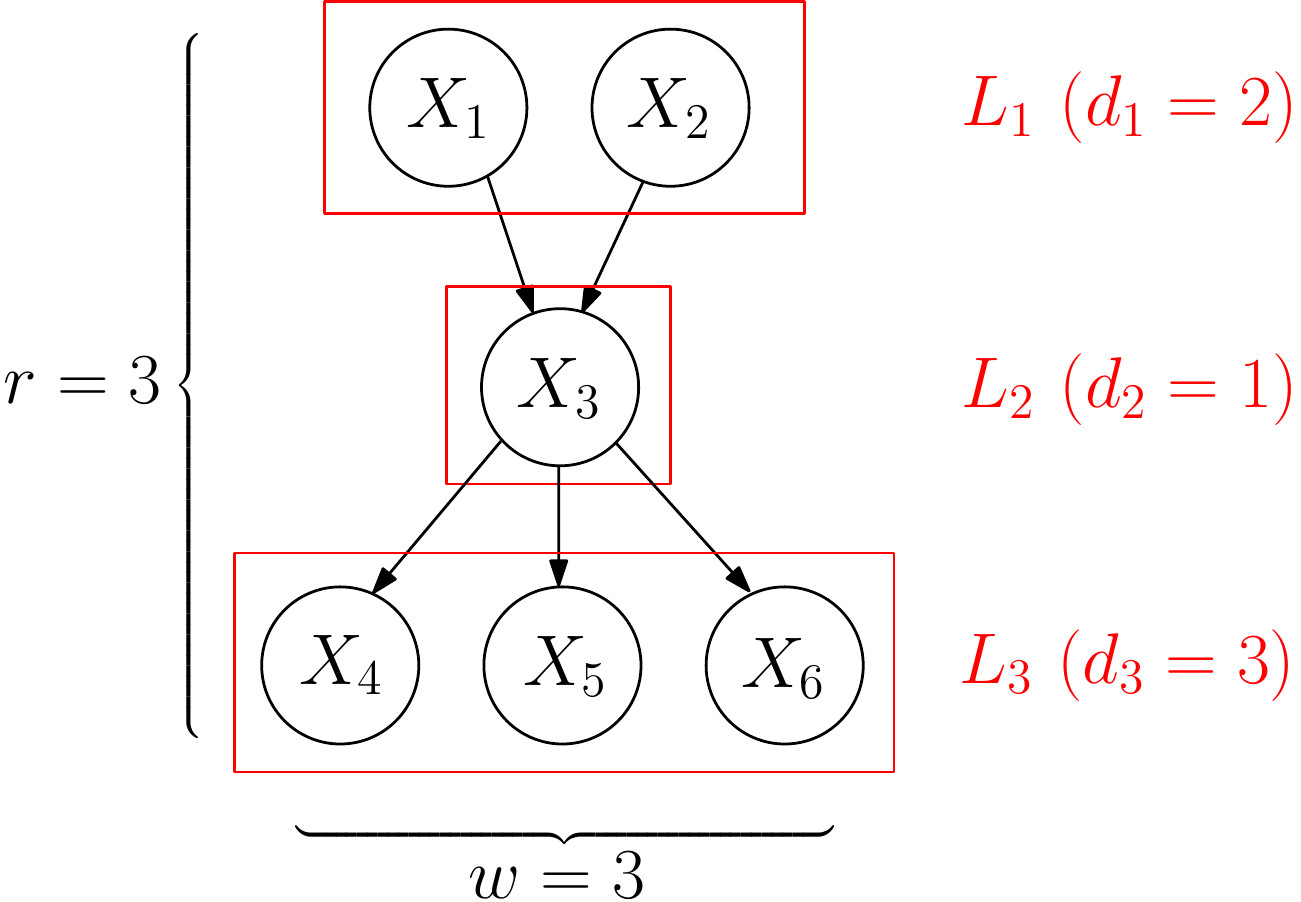}
\caption{}
\label{fig:intro:layers}
\end{subfigure}
\caption{(a) Existing methods may not find a correct topological ordering in simple settings when $d=3$. (b) Example of a layer decomposition $\layer(\gr)$ of a DAG on $d=6$ nodes.}
\label{fig:intro}
\end{figure}

\paragraph{Contributions}
Figure~\ref{fig:intro:camfail} illustrates a key motivation for our work: While there exist methods that obtain various statistical guarantees, they lack provably efficient algorithms, or vice versa. 
As a result, these methods can fail in simple settings. 
Our focus is on \emph{simultaneous} computational and statistical guarantees that are explicit and nonasymptotic in a model-free setting.
More specifically, our main contributions are as follows:

\begin{itemize}
\item We show that the algorithms of \citet{ghoshal2017ident} and \citet{chen2018causal} rigourously extend to a model-free setting, and provide a method-agnostic analysis of the resulting extension (Theorem~\ref{thm:main:sample}).  That is, the time and sample complexity bounds depend on the choice of estimator used, and this dependence is made explicit in the bounds (Section~\ref{sec:ident:alg}, Section~\ref{sec:sample}).
\item We prove that this algorithm runs in at most $O(nd^{5})$ time and needs at most $\Omega((d^{2}/\eps)^{1+\width/2})$ samples (Corollary~\ref{cor:main:sample}). 
Moreover, the exponential dependence on $d$ can be improved by imposing additional sparsity or smoothness assumptions, and can even be made polynomial (see Section~\ref{sec:sample} for discussion). This is an expected consequence of our estimator-agnostic approach.
\item We show how existing identifiability results based on ordering variances can be unified and generalized to include model-free families (Theorem~\ref{thm:gen:ident}, Section~\ref{sec:ident:nonpar}). 
\item We show that greedy algorithms such as those used in the CAM algorithm \citep{buhlmann2014} can provably fail to recover an identifiable DAG (Example~\ref{ex:cam:fail}), as shown in Figure~\ref{fig:intro:camfail} (Section~\ref{sec:ident:comparison}).
\item Finally, we run a simulation study to evaluate the resulting algorithm in a variety of settings against seven state-of-the-art algorithms (Section~\ref{sec:exp}).
\end{itemize}
Our simulation results can be summarized as follows: When implemented using generalized additive models \citep{hastie1990generalized}, our method outperforms most state-of-the-art methods, particularly on denser graphs with hub nodes. 
We emphasize here, however, that our main contributions lay in the theoretical analysis, specifically providing a polynomial-time algorithm with sample complexity guarantees.

\paragraph{Related work}
The literature on learning DAGs is vast, so we focus only on related work in the nonparametric setting.
The most closely related line work considers additive noise models (ANMs) \citep{peters2014,hoyer2009,buhlmann2014,chicharro2019conditionally}, and prove a variety of identifiability and consistency guarantees. Compared to our work, the identifiability results proved in these papers require that the structural equations are (a) nonlinear with (b) additive, independent noise. Crucially, these papers focus on (generally asymptotic) \emph{statistical} guarantees without any computational or algorithmic guarantees. 
There is also a closely related line of work for bivariate models \citep{mooij2014,monti2019causal,wu2020causal,mitrovic2018causal} as well as the post-nonlinear model \citep{zhang2009}.
\citet{huang2018generalized} proposed a greedy search algorithm using an RKHS-based generalized score, and proves its consistency assuming faithfulness. \citet{rothenhausler2018causal} study identifiability of a general family of partially linear models and prove consistency of a score-based search procedure in finding an equivalence class of structures. There is also a recent line of work on embedding neural networks and other nonparametric estimators into causal search algorithms \citep{lachapelle2019gradient,zheng2019learning,yu2019dag,ng2019masked,zhu2019causal} without theoretical guarantees. While this work was in preparation, we were made aware of the recent work \citep{park2020condvar} that proposes an algorithm that is similar to ours---also based on \citep{ghoshal2017ident} and \citep{chen2018causal}---and establishes its sample complexity for linear Gaussian models. 
In comparison to these existing lines of work, our focus is on simultaneous computational and statistical guarantees that are explicit and nonasymptotic (i.e. valid for all finite $d$ and $n$), for the fully nonlinear, nonparametric, and model-free setting.

\paragraph{Notation}
Subscripts (e.g. $X_{j}$) will always be used to index random variables and superscripts (e.g. $X_{j}^{(i)}$) to index observations.
For a matrix $W=(w_{kj})$, $w_{\cdot j}\in\R^{d}$ is the $j$th column of $W$.
We denote the indices by $[d]=\{1,\ldots,d\}$, and frequently abuse notation by identifying the indices $[d]$ with the random vector $X=(X_{1},\ldots,X_{d})$. 
For example, nodes $X_{j}$ are interchangeable with their indices $j$ (and subsets thereof), so e.g. $\var(j\given A)$ is the same as $\var(X_{j}\given X_{A})$.

\section{Background}
\label{sec:background}

Let $X=(X_{1},\ldots,X_{d})$ be a $d$-dimensional random vector and $\gr=(V,E)$ a DAG where we implicitly assume $V=X$. The \emph{parent set} of a node is defined as 
$\pa_{\gr}(X_{j})=\{i: (i,j)\in E\}$, 
or simply $\pa(j)$ for short.
A \emph{source} node is any node $X_{j}$ such that $\pa(j)=\emptyset$ and an \emph{ancestral set} is any set $A\subset V$ such that $X_{j}\in A\implies\pa(j)\subset A$. 
The graph $\gr$ is called a \emph{Bayesian network} (BN) for $X$ if it satisfies the Markov condition, i.e. that each variable is conditionally independent of its non-descendants given its parents.
Intuitively, a BN for $X$ can be interpreted as a representation of the direct and indirect relationships between the $X_{j}$, e.g. an edge $X_{i}\to X_{j}$ indicates that $X_{j}$ depends directly on $X_{i}$, and not vice versa. 
Under additional assumptions such as causal minimality and no unmeasured confounding, these arrows may be interpreted causally; for more details, see the surveys \citep{buhlmann2018invariance,scholkopf2019causality} or the textbooks \citep{lauritzen1996,koller2009,spirtes2000,pearl2009,peters2017elements}.

The goal of structure learning is to learn a DAG $\gr$ from i.i.d. observations $X^{(i)}\iid\pr(X)$. Throughout this paper, we shall exploit the following well-known fact: To learn $\gr$, it suffices to learn a topological sort of $\gr$, i.e. an ordering $\ord$ such that $X_{i}\to X_{j}\implies X_{i}\ord X_{j}$. A brief review of this material can be found in the supplement. 

\paragraph{Equal variances}
Recently, a new approach has emerged which was originally cast as an approach to learn equal variance DAGs \citep{ghoshal2017ident,chen2018causal}, although it has since been generalized beyond the equal variance case \citep{ghoshal2017sem,park2020identifiability,park2020condvar}. An equal variance DAG is a linear structural equation model (SEM) that satisfies 
\begin{align}
\label{eq:eqvar:lin}
X_{j}
= \ip{w_{\cdot j},X} + \err_{j},
\quad
\var(\err_{j}) 
= \sigma^{2},
\quad
\err_{j}\indep\pa(j),
\quad
w_{kj} = 0\iff k\notin\pa(j)
\end{align}
for some weights $w_{kj}\in\R$.
Under the model \eqref{eq:eqvar:lin}, a simple algorithm can learn the graph $\gr$ 
by first learning a topological sort $\ord$. For these models, we have the following decomposition of the variance:
\begin{align}
\label{eq:var:decomp:lin}
\var(X_{j})
&= \var(\ip{w_{\cdot j},X}) + \var(\err_{j}).
\end{align}
Thus, as long as $\var(\ip{w_{\cdot j},X})>0$, we have $\var(X_{j})>\var(\err_{j})$. It follows that as long as $\var(\err_{j})$ does not depend on $j$, it is possible to identify a source in $\gr$ by simply minimizing the residual variances. 
This is the essential idea behind algorithms based on equal variances in the linear setting \citep{ghoshal2017ident,chen2018causal}.
Alternatively, it is possible to iteratively identify best sinks by minimizing marginal precisions.
Moreover, this argument shows that the assumption of linearity is not crucial, and this idea can readily be extended to ANMs, as in \citep{park2020condvar}. Indeed, the crucial assumption in this argument is the independence of the noise $\err_{j}$ and the parents $\pa(X_{j})$; in the next section we show how these assumptions can be removed altogether.

\paragraph{Layer decomposition of a DAG}
Given a DAG $\gr$, define a collection of sets as follows: $\layer_{0}:=\emptyset$, $\anc_{j}=\cup_{m=0}^{j}\layer_{m}$ and for $j>0$, $\layer_{j}$ is the set of all source nodes in the subgraph $\gr[V-\anc_{j-1}]$ formed by removing the nodes in $\anc_{j-1}$. So, e.g., $\layer_{1}$ is the set of source nodes in $\gr$ and $\anc_{1}=\layer_{1}$. This decomposes $\gr$ into layers, where each layer $\layer_{j}$ consists of nodes that are sources in the subgraph $\gr[V-\anc_{j-1}]$, and $\anc_{j}$ is an ancestral set for each $j$. 
Let $r$ denote the number of ``layers'' in $\gr$, $\layer(\gr):=(\layer_{1},\ldots,\layer_{r})$ be the corresponding layers.
The quantity $r$ effectively measure the depth of a DAG. See Figure~\ref{fig:intro:layers} for an illustration.

Learning $\gr$ is equivalent to learning the sets $\layer_{1},\ldots,\layer_{r}$, since any topological sort $\pi$ of $\gr$ can be determined from $\layer(\gr)$, and from any sort $\pi$, the graph $\gr$ can be recovered via variable selection. Unlike a topological sort of $\gr$, which may not be unique, the layer decomposition $\layer(\gr)$ is always unique.
Therefore, without loss of generality, in the sequel we consider the problem of identifying and learning $\layer(\gr)$.

\section{Identifiability and algorithmic consequences}
\label{sec:ident}

This section sets the stage for our main results on learning nonparametric DAGs: First, we show that existing identifiability results for equal variances generalize to a family of model-free, nonparametric distributions. Second, we show that this motivates an algorithm very similar to existing algorithms in the equal variance case. We emphasize that various incarnations of these ideas have appeared in previous work \citep{ghoshal2017ident,chen2018causal,ghoshal2017sem,park2020identifiability,park2020condvar}, and our effort in this section is to unify these ideas and show that the same ideas can be applied in more general settings without linearity or independent noise.
Once this has been done, our main sample complexity result is presented in Section~\ref{sec:sample}.

\subsection{Nonparametric identifiability}
\label{sec:ident:nonpar}
In general, a BN for $X$ need not be unique, i.e. $\gr$ is not necessarily identifiable from $\pr(X)$. A common strategy in the literature to enforce identifiability is to impose structural assumptions on the conditional distributions $\pr(X_{j}\given \pa(j))$, for which there is a broad literature on identifiability. Our first result shows that identifiability is guaranteed as long as the residual variances $\E\var(X_{j}\given \pa(j))$ do not depend on $j$. This is a natural generalization of the notion of equality of variances for linear models \citep{peters2013,ghoshal2017ident,chen2018causal}. 
\begin{thm}
\label{thm:gen:ident}
If $\E\var(X_{j}\given \pa(j))\equiv\sigma^{2}$ does not depend on $j$, then $\gr$ is identifiable from $\pr(X)$.
\end{thm}
The proof of Theorem~\ref{thm:gen:ident} can be found in the supplement.
This result makes no structural assumptions on the local conditional probabilities $\pr(X_{j}\given \pa(j))$. To illustrate, we consider some examples below.

\begin{ex}[Causal pairs, \citep{mooij2014}]
Consider a simple model on two variables: $X\to Y$ with $\E\var(Y\given X)=\var(X)$. Then as long as $\E[Y\given X]$ is nonconstant,
Theorem~\ref{thm:gen:ident} implies the causal order is identifiable. No additional assumptions on the noise or functional relationships are necessary.
\end{ex}

\begin{ex}[Binomial models, \citep{park2017}]
Assume $X_{j}\in\{0,1\}$ and $X_{j}=\BernoulliDist(f_{j}(\pa(j)))$ with $f_{j}(\pa(j))\in[0,1]$. Then Theorem~\ref{thm:gen:ident} implies that if $\E f_{j}(\pa(j))(1-f_{j}(\pa(j)))\equiv\sigma^{2}$ does not depend on $j$, then $\gr$ is identifiable.
\end{ex}

\begin{ex}[Generalized linear models]
The previous example can of course be generalized to arbitrary generalized linear models: Assume $\pr[X_{j}\given \pa(j)]\propto \exp(X_{j}\theta_{j}-K(\theta_{j}))$, where $\theta_{j}=f_{j}(\pa(j))$ and $K(\theta_{j})$ is the partition function. Then Theorem~\ref{thm:gen:ident} implies that if $\E[K''(f_{j}(\pa(j)))]\equiv\sigma^{2}$ does not depend on $j$, then $\gr$ is identifiable.
\end{ex}

\begin{ex}[Additive noise models, \citep{peters2014}]
Finally, we observe that Theorem~\ref{thm:gen:ident} generalizes existing results for ANMs:
In an ANM, we have $X_{j}=f_{j}(\pa(j))+\err_{j}$ with $\err_{j}\independent\pa(j)$. If $\var(\err_{j})=\sigma^{2}$, then an argument similar to \eqref{eq:var:decomp:lin} shows that ANMs with equal variances are identifiable. 
Theorem~\ref{thm:gen:ident} applies to more general additive noise models $X_{j}=f_{j}(\pa(j))+g_{j}(\pa(j))^{1/2}\err_{j}$ with heteroskedastic, uncorrelated (i.e. not necessarily independent) noise.
\end{ex}

\paragraph{Unequal variances}
Early work on this problem focused on the case of equal variances \citep{ghoshal2017ident,chen2018causal}, as we have done here. This assumption illustrates the main technical difficulties in proving identifiability, and it is well-known by now that equality of variances is not necessary, and a weaker assumption that allows for heterogeneous residual variances suffices in special cases \citep{ghoshal2017sem,park2020identifiability}. Similarly, the extension of Theorem~\ref{thm:gen:ident} to such heterogeneous models is straightforward, and omitted for brevity; see Appendix~\ref{app:unequal} in the supplement for additional discussion and simulations. In the sequel, we focus on the case of equality for simplicity and ease of interpretation.

\subsection{A polynomial-time algorithm}
\label{sec:ident:alg}

The basic idea behind the top-down algorithm proposed in \citep{chen2018causal} can easily be extended to the setting of Theorem~\ref{thm:gen:ident}, and is outlined in Algorithm~\ref{alg:eqvaranm:pop}. The only modification is to replace the error variances $\var(\err_{j})=\sigma^{2}$ from the linear model \eqref{eq:eqvar:lin} with the corresponding residual variances (i.e. $\E\var(X_{\ell}\given S_{j})$), which are well-defined for any $\pr(X)$ with finite second moments.

A natural idea to translate Algorithm~\ref{alg:eqvaranm:pop} into an empirical algorithm is to replace the residual variances with an estimate based on the data. 
One might then hope to use similar arguments as in the linear setting to establish consistency and bound the sample complexity. Perhaps surprisingly, this does not work unless the topological sort of $\gr$ is unique. When there is more than one topological sort, it becomes necessary to uniformly bound the errors of all possible residual variances---and in the worst case there are exponentially many ($d2^{d-1}$ to be precise) possible residual variances. 
The key issue is that the sets $S_{j}$ in Algorithm~\ref{alg:eqvaranm:pop} are \emph{random} (i.e. data-dependent), and hence unknown in advance.
This highlights a key difference between our algorithm and existing work for linear models such as \citep{ghoshal2017ident,ghoshal2017sem,chen2018causal,park2020condvar}: In our setting, the residual variances cannot be written as simple functions of the covariance matrix $\Sigma:=\E X\!X^{T}$, which simplifies the analysis for linear models considerably. 
Indeed, although the same exponential blowup arises for linear models, in that case consistent estimation of the covariance matrix $\Sigma:=\E X\!X^{T}$ provides \emph{uniform} control over all possible residual variances (e.g., see Lemma~6 in \citep{chen2018causal}). In the nonparametric setting, this reduction no longer applies. 

To get around this technical issue, we modify Algorithm~\ref{alg:eqvaranm:pop} to learn $\gr$ one layer $\layer_{j}$ at a time, as outlined in Algorithm~\ref{alg:eqvaranm:emp} (see Section~\ref{sec:background} for details on $\layer_{j}$). As a result, we need only estimate $\sigma_{\ell j}^{2}:=\E\var(X_{\ell}\given \anc_{j})$, which involves regression problems with at most $|\anc_{j}|$ nodes. 
We use the plug-in estimator \eqref{eq:plugin} for this, although more sophisticated estimators are available \citep{doksum1995nonparametric,robins2008higher}. This also necessitates the use of sample splitting in Step 3(a) of Algorithm~\ref{alg:eqvaranm:emp}, which is necessary for the theoretical arguments but not needed in practice.

The overall computational complexity of Algorithm~\ref{alg:eqvaranm:emp}, which we call \NPVAR{}, is $O(ndrT)$, where $T$ is the complexity of computing each nonparametric regression function $\wh{f}_{\ell j}$. For example, if a kernel smoother is used, $T=O(d^{3})$ and thus the overall complexity is $O(nrd^{4})$. For comparison, an oracle algorithm that knows the true topological order of $\gr$ in advance would still need to compute $d$ regression functions, and hence would have complexity $O(dT)$. Thus, the extra complexity of learning the topological order is only $O(nr)=O(nd)$, which is linear in the dimension and the number of samples. Furthermore, under additional assumptions on the sparsity and/or structure of the DAG, the time complexity can be reduced further, however, our analysis makes no such assumptions.

\begin{algorithm}[t]
\caption{Population algorithm for learning nonparametric DAGs}
\label{alg:eqvaranm:pop}
\begin{enumerate}
\item Set $S_{0}=\emptyset$ and for $j=0,1,2,\ldots$, let
\begin{align*}
k_{j} 
&=\argmin_{\ell\notin S_{j}}\E\var(X_{\ell}\given S_{j}), \qquad
S_{j+1} 
= S_{j} \cup \{k_{j}\}.
\end{align*}
\item Return the DAG $\gr$ that corresponds to the topological sort $(k_{1},\ldots,k_{d})$.
\end{enumerate}
\end{algorithm}

\begin{algorithm}[t]
\caption{\NPVAR{} algorithm}
\label{alg:eqvaranm:emp}
\textbf{Input:} $X^{(1)},\ldots,X^{(n)}$, $\eta>0$.
\begin{enumerate}
\item Set $\wh{\layer}_{0}=\emptyset$, $\wh{\sigma}_{\ell 0}^{2}=\wh{\var}(X_{\ell})$, $k_{0}=\argmin_{\ell}\wh{\sigma}_{\ell 0}^{2}$, $\wh{\sigma}_{0}^{2}=\sigma_{k_{0}0}^{2}$.
\item Set $\wh{\layer}_{1}:=\{\ell : |\wh{\sigma}_{\ell 0}^{2} - \wh{\sigma}_{0}^{2}| < \eta\}$.
\item For $j=2,3,\ldots$: 
\begin{enumerate}
\item Randomly split the $n$ samples in half and let $\wh{\anc}_{j}:=\cup_{m=1}^{j}\wh{\layer}_{m}$.
\item For each $\ell\notin\wh{\anc}_{j}$, use the first half of the sample to estimate $f_{\ell j}(X_{\wh{\anc}_{j}})=\E[X_{\ell}\given \wh{\anc}_{j}]$ via a nonparametric estimator $\wh{f}_{\ell j}$.
\item For each $\ell\notin\wh{\anc}_{j}$, use the second half of the sample to estimate the residual variances via the plug-in estimator
\begin{align}
\label{eq:plugin}
\wh{\sigma}_{\ell j}^{2}
&= \frac1{n/2}\sum_{i=1}^{n/2}(X_{\ell}^{(i)})^{2} - \frac1{n/2}\sum_{i=1}^{n/2}\wh{f}_{\ell j}(X_{\wh{\anc}_{j}}^{(i)})^{2}.
\end{align}
\item Set $k_{j} =\argmin_{\ell\notin\wh{\anc}_{j}}\wh{\sigma}_{\ell j}^{2}$ and $\wh{\layer}_{j+1} = \{\ell : |\wh{\sigma}_{\ell j}^{2} - \wh{\sigma}_{k_{j}j}^{2}| < \eta, \,\ell\notin\wh{\anc}_{j}\}$.
\end{enumerate}
\item Return $\wh{\layer}=(\wh{\layer}_{1},\ldots,\wh{\layer}_{\wh{r}})$.
\end{enumerate}
\end{algorithm}

\subsection{Comparison to existing algorithms}
\label{sec:ident:comparison}

Compared to existing algorithms based on order search and equal variances, \NPVAR{} applies to more general models without parametric assumptions, independent noise, or additivity.
It is also instructive to make comparisons with greedy score-based algorithms such as causal additive models (CAM, \citep{buhlmann2014}) and greedy DAG search (GDS, \citep{peters2013}). We focus here on CAM since it is more recent and applies in nonparametric settings, however, similar claims apply to GDS as well.

CAM is based around greedily minimizing the log-likelihood score for additive models with Gaussian noise. In particular, it is not guaranteed to find a global minimizer, which is as expected since it is based on a nonconvex program. This is despite the global minimizer---if it can be found---having good statistical properties.
The next example shows that, in fact, there are identifiable models for which CAM will find the wrong graph with high probability.

\begin{ex}%
\label{ex:cam:fail}
Consider the following three-node additive noise model with $\err_{j}\sim\normalN(0,1)$:
\begin{align}
\label{eq:cam:fail}
\begin{aligned}
X_{1} 
&= \err_{1}, \\
X_{2}
&= g(X_{1}) + \err_{2}, \\
X_{3}
&= g(X_{1}) + g(X_{2}) + \err_{3}.
\end{aligned}
\end{align}
In the supplement (Appendix~\ref{app:cam}), we show the following: \emph{There exist infinitely many nonlinear functions $g$ for which the CAM algorithm returns an incorrect order under the model \eqref{eq:cam:fail}.}
This is illustrated empirically in Figure~\ref{fig:intro:camfail} for the nonlinearities $g(u)=\text{sgn}(u)|u|^{1.4}$ and $g(u)=\sin u$. In each of these examples, the model satisfies the identifiability conditions for CAM as well as the conditions required in our work.
\end{ex}
We stress that this example does not contradict the statistical results in \citet{buhlmann2014}: It only shows that the \emph{algorithm} may not find a global minimizer and as a result, returns an incorrect variable ordering. Correcting this discrepancy between the algorithmic and statistical results is a key motivation behind our work. In the next section, we show that \NPVAR{} provably learns the true ordering---and hence the true DAG---with high probability.

\section{Sample complexity}
\label{sec:sample}

Our main result analyzes the sample complexity of \NPVAR{} (Algorithm~\ref{alg:eqvaranm:emp}). Recall the layer decomposition $\layer(\gr)$ from Section~\ref{sec:background} and define $d_{j}:=|\anc_{j}|$. Let $f_{\ell j}(X_{\anc_{j}})=\E[X_{\ell}\given \anc_{j}]$. 

\begin{cond}[Regularity]
\label{cond:reg}
For all $j$ and all $\ell\notin \anc_{j}$,
(a)~$X_{j}\in[0,1]$, 
(b)~$f_{\ell j}:[0,1]^{d_{j}}\to[0,1]$, 
(c)~$f_{\ell j}\in L^{\infty}([0,1]^{d_{j}})$, and 
(d)~$\var(X_{\ell}\given \anc_{j}) \le \zeta_{0}<\infty$.
\end{cond}
These are the standard regularity conditions from the literature on nonparametric statistics~\citep{gyorfi2006distribution,tsybakov2009introduction}, and can be weakened (e.g. if the $X_{j}$ and $f_{\ell j}$ are unbounded, see \citep{kohler2009optimal}). We impose these stronger assumptions in order to simplify the statements and focus on technical details pertinent to graphical modeling and structure learning. The next assumption is justified by Theorem~\ref{thm:gen:ident}, and as we have noted, can also be weakened.
\begin{cond}[Identifiability]
\label{cond:ident}
$\E\var(X_{j}\given \pa(j))\equiv\sigma^{2}$ does not depend on $j$.
\end{cond}
Our final condition imposes some basic finiteness and consistency requirements on the chosen nonparametric estimator $\wh{f}$, which we view as a function for estimating $\E[Y\given Z]$ from an arbitrary distribution over the pair $(Y,Z)$.
\begin{cond}[Estimator]
\label{cond:est}
The nonparametric estimator $\wh{f}$ satisfies (a) $\E[Y\given Z]\in L^{\infty}\implies \wh{f}\in L^{\infty}$ and (b) $\E_{\wh{f}}\norm{\wh{f}(Z) - \E[Y\given Z]}_{2}^{2}\to 0$.
\end{cond}
This is a mild condition that is satisfied by most popular estimators including kernel smoothers, nearest neighbours, and splines, and in particular, Condition~\ref{cond:est}(a) is only used to simplify the theorem statement and can easily be relaxed.

\begin{thm}
\label{thm:main:sample}
Assume Conditions~\ref{cond:reg}-\ref{cond:est}.
Let $\Delta_{j}>0$ be such that $\E\var(X_{\ell}\given \anc_{j})>\sigma^{2}+\Delta_{j}$ for all $\ell\notin \anc_{j}$ and define $\Delta:=\inf_{j}\Delta_{j}$. 
Let $\npestbd^{2}:=\sup_{\ell,j}\E_{\wh{f}_{\ell j}}\norm{f_{\ell j}(X_{\anc_{j}})-\wh{f}_{\ell j}(X_{\anc_{j}})}_{2}^{2}$.
Then for any $\npestbd\sqrt{d}<\eta<\Delta/2$,
\begin{align}
\label{eq:thm:main:sample}
\pr(\wh{\layer} = \layer(\gr))
&\ges 1 - \frac{\npestbd^{2}}{\eta^{2}}rd%
\end{align}
\end{thm}
Once the layer decomposition $\layer(\gr)$ is known, the graph $\gr$ can be learned via standard nonlinear variable selection methods (see Appendix~\ref{app:order} in the supplement).

A feature of this result is that it is agnostic to the choice of estimator $\wh{f}$, as long as it satisfies Condition~\ref{cond:est}. The dependence on $\wh{f}$ is quantified through $\npestbd^{2}$, which depends on the sample size $n$ and represents the rate of convergence of the chosen nonparametric estimator. Instead of choosing a specific estimator, Theorem~\ref{thm:main:sample} is stated so that it can be applied to general estimators. As an example, suppose each $f_{\ell j}$ is Lipschitz continuous and $\wh{f}$ is a standard kernel smoother. Then
\begin{align*}
\E_{\wh{f}_{\ell j}}\norm{f_{\ell j}(X_{\layer_{j}})-\wh{f}_{\ell j}(X_{\layer_{j}})}_{2}^{2}
\le \npestbd^{2}
\les n^{-\tfrac{2}{2+\width}}.
\end{align*}
Thus we have the following special case:
\begin{cor}
\label{cor:main:sample}
Assume each $f_{\ell j}$ is Lipschitz continuous. Then $\wh{\layer}$ can be computed in $O(nd^{5})$ time and $\pr(\wh{\layer} = \layer(\gr))\ge1-\eps$ as long as $n=\Omega((rd/(\eta^{2}\eps))^{1+d/2})$.
\end{cor}
This is the best possible rate attainable by any algorithm without imposing stronger regularity conditions (see e.g. \sec5 in \citep{gyorfi2006distribution}). Furthermore, $\delta^{2}$ can be replaced with the error of an arbitrary estimator of the residual variance itself (i.e. something besides the plug-in estimator \eqref{eq:plugin}); see Proposition~\ref{prop:sample:bound} in Appendix~\ref{app:proof:main} for details.

To illustrate these results, consider the problem of finding the direction of a Markov chain $X_{1}\to X_{2}\to\cdots\to X_{d}$ whose transition functions $\E[X_{j}\given X_{j-1}]$ are each Lipschitz continuous. Then $r=d$, so Corollary~\ref{cor:main:sample} implies that $n=\Omega((d^{2}/(\eta\sqrt{\eps}))^{1+d/2})$ samples are sufficient to learn the order---and hence the graph as well as each transition function---with high probability. Since $r=d$ for any Markov chain, this particular example maximizes the dependence on $d$; at the opposite extreme a bipartite graph with $r=2$ would require only $n=\Omega((\sqrt{d}/(\eta\sqrt{\eps}))^{1+d/2})$. In these lower bounds, it is not necessary to know the type of graph (e.g. Markov chain, bipartite) or the depth $r$.

\paragraph{Choice of $\eta$}
The lower bound $\eta>\npestbd\sqrt{d}$ is not strictly necessary, and is only used to simplify the lower bound in \eqref{eq:thm:main:sample}. In general, taking $\eta$ sufficiently small works well in practice. The main tradeoff in choosing $\eta>0$ is computational: A smaller $\eta$ may lead to ``splitting'' one of the layers $\layer_{j}$. In this case, \NPVAR{} still recovers the structure correctly, but the splitting results in redundant estimation steps in Step 3 (i.e. instead of estimating $\layer_{j}$ in one iteration, it takes multiple iterations to estimate correctly). The upper bound, however, is important: If $\eta$ is too large, then we may include spurious nodes in the layer $\layer_{j}$, which would cause problems in subsequent iterations.

\paragraph{Nonparametric rates}
Theorem~\ref{thm:main:sample} and Corollary~\ref{cor:main:sample} make no assumptions on the sparsity of $\gr$ or smoothness of the mean functions $\E[X_{\ell}\given \anc_{j}]$. For this reason, the best possible rate for a na\"ive plug-in estimator of $\E\var(X_{\ell}\given \anc_{j})$ is bounded by the minimax rate for estimating $\E[X_{\ell}\given \anc_{j}]$. For practical reasons, we have chosen to focus on an agnostic analysis that does not rely on any particular estimator. Under additional sparsity and smoothness assumptions, these rates can be improved, which we briefly discuss here. 

For example, by using adaptive estimators such as RODEO \citep{lafferty2008rodeo} or GRID \citep{giordano2020grid}, the sample complexity will depend only on the sparsity of $f_{\ell j}(X_{A_j})$, i.e. $d^{*}=\max_j\max_{\ell\notin A_j}|\{k \in A_j : \partial_k f_{\ell j} \ne 0 \}|$, where $\partial_k$ is the $k$th partial derivative.
Another approach that does not require adaptive estimation is to assume $|L_{j}|\le w$ and define $r^{*} := \sup\{ |i-j| : e=(e_{1},e_{2})\in E, e_{1}\in L_{i}, e_{2}\in L_{j}\}$.
Then $\npestbd^{2}\asymp n^{-2/(2+wr^{*})}$, and the resulting sample complexity depends on $wr^{*}$ instead of $d$. 
For a Markov chain with $w=r^{*}=1$ this leads to a substantial improvement.

Instead of sparsity, we could impose stronger smoothness assumptions: Let $\beta_{*}$ denote the smallest H\"older exponent of any $f_{\ell j}$. Then if $\beta_{*}\ge d/2$, then one can use a one-step correction to the plug-in estimator \eqref{eq:plugin} to obtain a root-$n$ consistent estimator of $\E\var(X_{\ell}\given \anc_{j})$ \citep{robins2008higher,kandasamy2015nonparametric}. Another approach is to use undersmoothing \citep{doksum1995nonparametric}. In this case, the exponential sample complexity improves to polynomial sample complexity. For example, in Corollary~\ref{cor:main:sample}, if we replace Lipschitz with the stronger condition that $\beta_{*}\ge d/2$, then the sample complexity improves to $n=\Omega(rd/(\eta^{2}\eps))$.

\section{Experiments}
\label{sec:exp}

Finally, we perform a simulation study to compare the performance of \NPVAR{} against state-of-the-art methods for learning nonparametric DAGs. The algorithms are: 
\RESIT~\citep{peters2014}, 
\CAM~\citep{buhlmann2014}, 
\EqVar~\citep{chen2018causal},
\NOTEARS~\citep{zheng2019learning},
\GSGES~\citep{huang2018generalized},
\PC~\citep{spirtes1991}, and
\GES~\citep{chickering2003}.
In our implementation of \NPVAR{}, we use generalized additive models (GAMs) for both estimating $\wh{f}_{\ell j}$ and variable selection. 
One notable detail is our implementation of \EqVar{}, which we adapted to the nonlinear setting by using GAMs instead of subset selection for variable selection (the order estimation step remains the same).
Full details of the implementations used as well as additional experiments can be found in the supplement. 
Code implementing the \NPVAR{} algorithm is publicly available at \url{https://github.com/MingGao97/NPVAR}.

\begin{figure}[t]
\centering
\begin{subfigure}[t]{0.9\textwidth}
\includegraphics[width=1.\textwidth]{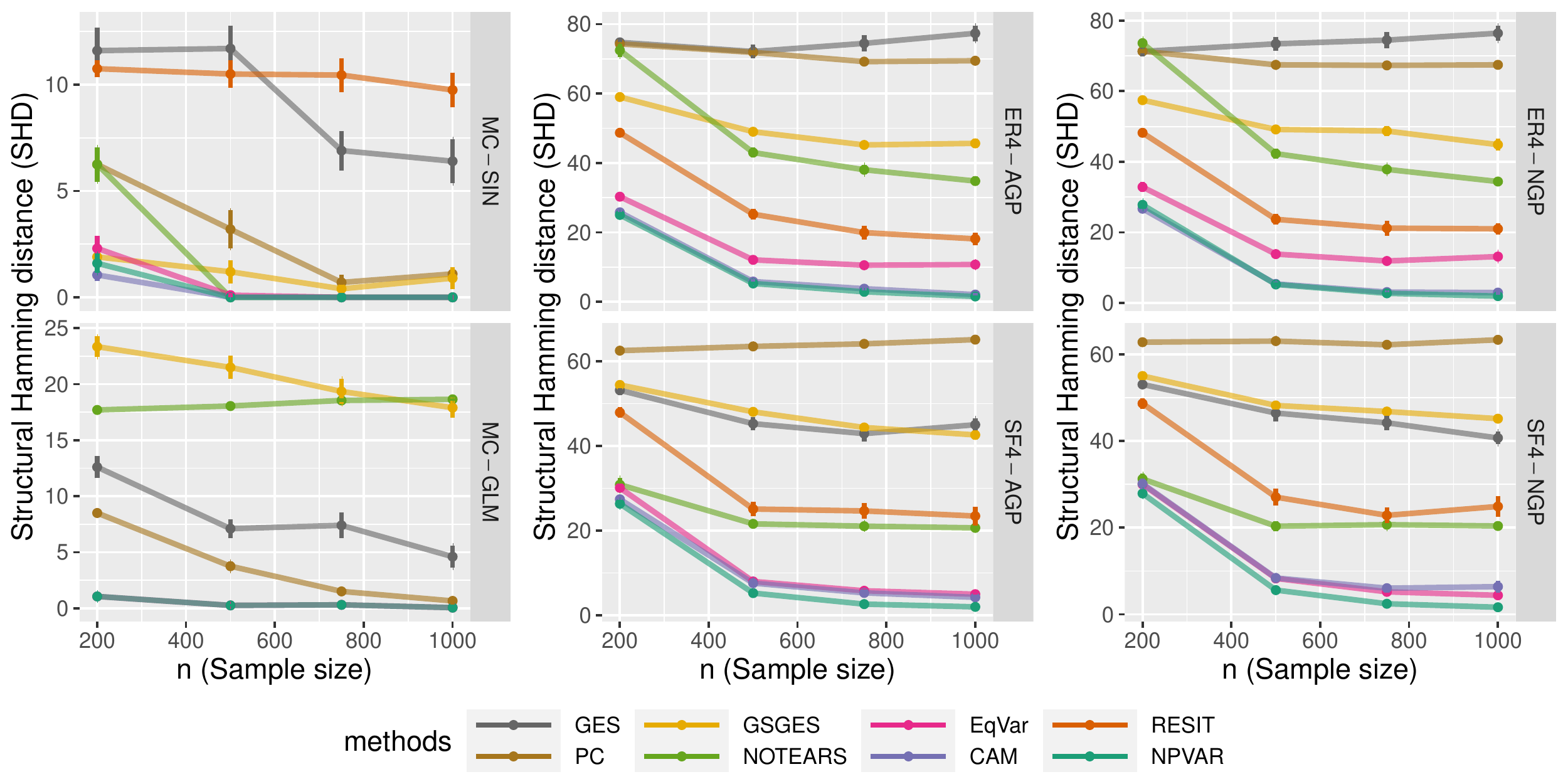}
\caption{SHD vs. $n$ ($d=20$).}
\label{fig:exp:shd_vs_n}
\end{subfigure}
\\
\begin{subfigure}[t]{0.9\textwidth}
\includegraphics[width=1.\textwidth]{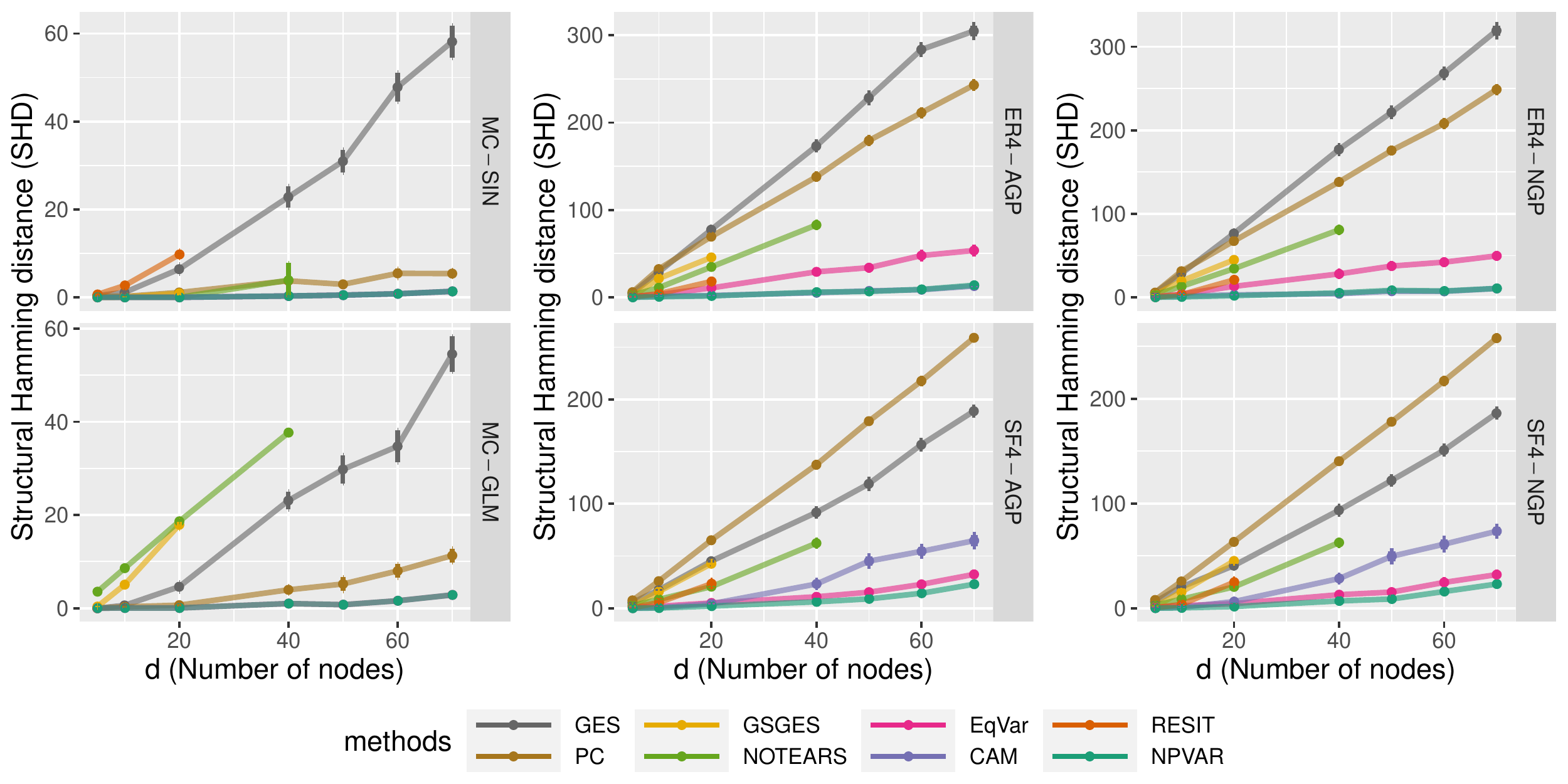}
\caption{SHD vs. $d$ ($n=1000$).}
\label{fig:exp:shd_vs_d}
\end{subfigure}
\caption{Structural Hamming distance (SHD) as a function of sample size ($n$) and number of nodes ($d$). Error bars denote $\pm$1 standard error. Some algorithms were only run for sufficiently small graphs due to high computational cost.}
\label{fig:exp:shd}
\end{figure}

We conducted a series of simulation on different graphs and models, comparing the performance in both order recovery and structure learning. Due to space limitations, only the results for structure learning in the three most difficult settings are highlighted in Figure~\ref{fig:exp:shd}. These experiments correspond to non-sparse graphs with non-additive dependence given by either a Gaussian process (GP) or a generalized linear model (GLM):
\begin{itemize}
\item \emph{Graph types.} We sampled three families of DAGs: Markov chains (MC), Erd\"os-R\'enyi graphs (ER), and scale-free graphs (SF). For MC graphs, there are exactly $d$ edges, whereas for ER and SF graphs, we sample graphs with $kd$ edges on average. This is denoted by ER4/SF4 for $k=4$ in Figure~\ref{fig:exp:shd}. Experiments on sparser DAGs can be found in the supplement.
\item \emph{Probability models.} For the Markov chain models, we used two types of transition functions: An additive sine model with $\pr(X_{j}\given X_{j-1})=\normalN(\sin(X_{j-1}),\sigma^{2})$ and a discrete model (GLM) with $X_{j}\in\{0,1\}$ and $\pr(X_{j}\given X_{j-1})\in\{p,1-p\}$. For the ER and SF graphs, we sampled $\E[X_{j}\given\pa(j)]$ from both additive GPs (AGP) and non-additive GPs (NGP).
\end{itemize}
Full details as well as additional experiments on order recovery, additive models, sparse graphs, and misspecified models can be found in the supplement (Appendix~\ref{app:exp}). 

\paragraph{Structure learning}
To evaluate overall performance, we computed the structural Hamming distance (SHD) between the learned DAG and the true DAG. SHD is a standard metric used for comparison of graphical models.
According to this metric, the clear leaders are \NPVAR{}, \EqVar{}, and \CAM{}. Consistent with existing results, existing methods tend to suffer as the edge density and dimension of the graphs increase, however, \NPVAR{} is more robust in these settings. Surprisingly, the \CAM{} algorithm remains quite competitive for non-additive models, although both \EqVar{} and \NPVAR{} clearly outperform \CAM{}. On the GLM model, which illustrates a non-additive model with non-additive noise, \EqVar{} and \NPVAR{} performed the best, although \PC{} showed good performance with $n=1000$ samples. Both \CAM{} and \RESIT{} terminated with numerical issues on the GLM model.

These experiments serve to corroborate our theoretical results and highlight the effectiveness of the \NPVAR{} algorithm, but of course there are tradeoffs. For example, algorithms such as CAM which exploit sparse and additive structure perform very well in settings where sparsity and additivity can be exploited, and indeed outperform \NPVAR{} in some cases. Hopefully, these experiments can help to shed some light on when various algorithms are more or less effective.

\paragraph{Misspecification and sensitivity analysis}
We also considered two cases of misspecification: In Appendix~\ref{app:unequal}, we consider an example where Condition~\ref{cond:ident} fails, but \NPVAR{} still successfully recovers the true ordering. This experiment corroborates our claims that this condition can be relaxed to handle unequal residual variances. We also evaluated the performance of \NPVAR{} on linear models as in \eqref{eq:eqvar:lin}, and in all cases it was able to recover the correct ordering.

\section{Discussion}
\label{sec:disc}

In this paper, we analyzed the sample complexity of a polynomial-time algorithm for estimating nonparametric causal models represented by a DAG. Notably, our analysis avoids many of the common assumptions made in the literature. Instead, we assume that the residual variances are equal, similar to assuming homoskedastic noise in a standard nonparametric regression model. Our experiments confirm that the algorithm, called \NPVAR{}, is effective at learning identifiable causal models and outperforms many existing methods, including several recent state-of-the-art methods. Nonetheless, existing algorithms such as CAM are quite competitive and apply in settings where NPVAR does not.

We conclude by discussing some limitations and directions for future work.
Although we have relaxed many of the common assumptions made in the literature, these assumptions have been replaced by an assumption on the residual variances that may not hold in practice. An interesting question is whether or not there exist provably polynomial-time algorithms for nonparametric models in under less restrictive assumptions.
Furthermore, although the proposed algorithm is polynomial-time, the worst-case $O(d^{5})$ dependence on the dimension is of course limiting. This can likely be reduced by developing more efficient estimators of the residual variance that do not first estimate the mean function. This idea is common in the statistics literature, however, we are not aware of such estimators specifically for the residual variance (or other nonlinear functionals of $\pr(X)$). 
Furthermore, our general approach can be fruitfully applied to study various parametric models that go beyond linear models, for which both computation and sample efficiency would be expected to improve. These are interesting directions for future work.

\paragraph{Acknowledgements}
We thank the anonymous reviewers for valuable feedback, as well as Y. Samuel Wang and Edward H. Kennedy for helpful discussions. B.A. acknowledges the support of the NSF via IIS-1956330 and the Robert H. Topel Faculty Research Fund. Y.D.'s work has been partially supported by the NSF (CCF-1439156, CCF-1823032, CNS-1764039).

\appendix

\section{Reduction to order search}
\label{app:order}

The fact that DAG learning can be reduced to learning a topological sort is well-known. For example, this fact is the basis of exact algorithms for score-based learning based on dynamic programming \cite{silander2012,ott2004,perrier2008,singh2005} as well as recent algorithms for linear models \citep{chen2018causal,ghoshal2017ident,ghoshal2017sem}. See also \citep{shojaie2010}. This fact has also been exploited in the overdispersion scoring model developed by \citet{park2017} as well as for nonlinear additive models \citep{buhlmann2014}. In fact, more can be said: Any ordering defines a minimal I-map of $\pr(X)$ via a simple iterative algorithm (see \sec3.4.1, Algorithm~3.2 in \citep{koller2009}), and this minimal I-map is unique as long as $\pr(X)$ satisfies the intersection property. This is guaranteed, for example, if $\pr(X)$ has a positive density, but holds under weaker conditions (see \citep{peters2015intersection} for necessary and sufficient conditions assuming $\pr(X)$ has a density and \citep{dawid1980}, Theorem~7.1, for the general case). This same algorithm can then be used to reconstruct the true DAG $\gr$ from the true ordering $\prec$. As noted in Section~\ref{sec:background}, a further reduction can be obtained by considering the layer decomposition $\layer(\gr)$, from which all topological orders $\prec$ of $\gr$ can be deduced.

Once the ordering is known, existing nonlinear variable selection methods \citep{lafferty2008rodeo,rosasco2013nonparametric,giordano2020grid,miller2010local,bertin2008selection,comminges2011tight} suffice to learn the parent sets $\pa(j)$ and hence the graph $\gr$. More specifically, given an order $\prec$, to identify $\pa(j)$, let $f(S_{j}):=\E[X_{j}\given S_{j}]$, where $S_{j}:=\{X_{k}:X_{k}\prec X_{j}\}$. The parent set of $X_{j}$ is given by the active variables in this conditional expectation, i.e. $\pa(j)=\{k:\partial_{k}f \ne 0\}$, where $\partial_{k}$ is the partial derivative of $f$ with respect to the $k$th argument.

In our experiments, we use exactly this procedure to learn $\gr$ from the order $\prec$, based on the data. Specifically, we use generalized additive models, similar to the pruning step in \citep{buhlmann2014}. See Appendix~\ref{app:exp} for more details.

\section{Proof of Theorem~\ref{thm:gen:ident}}
\label{app:proof:ident}

The key lemma is the following, which is easy to prove for additive noise models via \eqref{eq:var:decomp:lin}, and which we show holds more generally in non-additive models:
\begin{lemma}
\label{lem:gen:ident}
Let $A\subset V$ be an ancestral set in $\gr$. If $\E\var(X_{j}\given \pa(j))\equiv\sigma^{2}$ does not depend on $j$, then for any $j\notin A$,
\begin{align*}
\E\var(X_{j}\given X_{A}) &= \sigma^{2} \quad\text{if $\pa(j)\subset A$,}\\
\E\var(X_{j}\given X_{A}) &> \sigma^{2}  \quad\text{otherwise.}
\end{align*}
\end{lemma}

\begin{proof}[Proof of Lemma~\ref{lem:gen:ident}]
Let $B_{j}=\pa(X_{j})$, $\wb{B_{j}}:=B_{j}-A$, 
and $\wb{\gr}$ be the subgraph of $\gr$ formed by removing the nodes in the ancestral set $A$. Then 
\begin{align*}
\var(X_{j}\given X_{A})
&= \E[\var(X_{j}\given X_{A}, X_{B_{j}})\given X_{A}]
+ \var[\E[X_{j}\given X_{A}, X_{B_{j}}]\given X_{A}] \\
&= \E[\var(X_{j}\given X_{A}, X_{\wb{B_{j}}})\given X_{A}]
+ \var[\E[X_{j}\given X_{A}, X_{\wb{B_{j}}}]\given X_{A}].
\end{align*}
There are two cases: (i) $\wb{B_{j}}=\emptyset$, and (ii) $\wb{B_{j}}\ne\emptyset$. In case (i), it follows that $B_{j}\subset A$ and hence $\pa(j)\subset A$. Since $X_{j}$ is conditionally independent of its nondescendants (e.g. ancestors) given its parents, it follows that $\var(X_{j}\given X_{A})=\var(X_{j}\given X_{B_{j}})$ and hence
\begin{align*}
\E\var(X_{j}\given X_{A})
=\E\var(X_{j}\given X_{B_{j}})
=\sigma^{2}.
\end{align*}

In case (ii), it follows that 
\begin{align*}
\var(X_{j}\given X_{A}) 
&= \E[\var(X_{j}\given X_{A}, X_{\wb{B_{j}}})\given X_{A}]
+ \var[\E[X_{j}\given X_{A}, X_{\wb{B_{j}}}]\given X_{A}] \\
&= \E[\var(X_{j}\given X_{B_{j}})\given X_{A}]
+ \var[\E[X_{j}\given X_{B_{j}}]\given X_{A}],
\end{align*}
where again we used that $X_{j}$ is conditionally independent of its nondescendants (e.g. ancestors) given its parents to replace conditioning on $(X_{A}, X_{\wb{B_{j}}})=X_{A\cup B_{j}}$ with conditioning on $B_{j}$. 

Now suppose $X_{k}$ is in case (i) and $X_{j}$ is in case (ii). We wish to show that $\E\var(X_{j}\given X_{A})  > \E\var(X_{k}\given X_{A}) = \sigma^{2}$. Then
\begin{align*}
\E\var(X_{j}\given X_{A}) 
&= \E\big[\E[\var(X_{j}\given X_{B_{j}})\given X_{A}]\big]
+ \E\var[\E[X_{j}\given X_{B_{j}}]\given X_{A}] \\
&> \E\big[\E[\var(X_{j}\given \pa(j))\given X_{A}]\big] \\
&= \E\var(X_{j}\given \pa(j))\\
&= \E\var(X_{k}\given \pa(k)) \\
&= \sigma^{2},
\end{align*}
where we have invoked the assumption that $\E\var(X_{j}\given \pa(j))$ does not depend on $j$ to conclude $\E[\var(X_{j}\given \pa(j))\given X_{A}] = \E[\var(X_{k}\given \pa(k))\given X_{A}]$. 
This completes the proof.
\end{proof}

Theorem~\ref{thm:gen:ident} is an immediate corollary of Lemma~\ref{lem:gen:ident}. For completeness, we include a proof below.
\begin{proof}[Proof of Theorem~\ref{thm:gen:ident}]
Let $S(\gr)$ denote the set of sources in $\gr$ and note that Lemma~\ref{lem:gen:ident} implies that if $X_{s}\in S(\gr)$, then $\var(X_{s})<\var(X_{j})$ for any $s\ne j$. Thus, $S(\gr)$ is identifiable. Let $\gr_{1}$ denote the subgraph of $\gr$ formed by removing the nodes in $S(\gr)$. Since $S(\gr)=\layer_{1}=\anc_{1}$, $S(\gr)$ is an ancestral set in $\gr$. After conditioning on $\anc_{1}$, we can thus apply Lemma~\ref{lem:gen:ident} once again to identify the sources in $\gr_{1}$, i.e. $S(\gr_{1})=\layer_{2}$. By repeating this procedure, we can recursively identify $\layer_{1},\ldots,\layer_{r}$, and hence any topological sort of $\gr$.
\end{proof}

\subsection{Generalization to unequal variances}
\label{app:unequal}

In this appendix, we illustrate how Theorem~\ref{thm:gen:ident} can be extended to the case where residual variances are different, i.e. $\sigma_{j}^{2}=\E\var(X_{j} | \pa(j))$ is not independent of $j$. Let $\de(i)$ be the descendant of node $i$ and $[a:b]=\{a,a+1,\ldots,b-1,b\}$. Note also that for any nodes $X_u$ and $X_v$ in the same layer $\layer_m$ of the graph, if we interchange the position of $u$ and $v$ in some true ordering $\pi$ consistent with the graph to get $\pi_u$ and $\pi_v$, both $\pi_u$ and $\pi_v$ are correct orderings.

The following result is similar to existing results on unequal variances \citep{ghoshal2017sem,park2020identifiability,park2020condvar}, with the exception that it applies to general DAGs without linearity, additivity, or independent noise.

\begin{thm}
\label{thm:ident:unequal}
Suppose there exists an ordering $\pi$ such that for all $j\in [1:d]$ and $k\in \pi_{[j+1:d]}$, the following conditions holds: 
\begin{enumerate}
\item If $i=\pi_j$ and $k$ are not in the same layer $\layer_m$, then
\begin{align}
\label{eq:thm:ident:unequal}
    \sigma_i^2
    &<\sigma_k^2+\E\var(\E(X_k\given\pa(k))\vert X_{\pi_{[1:j-1]}}).
\end{align}
\item If $i$ and $k$ are in the same layer $\layer_m$, then either $\sigma^2_i= \sigma^2_k$ or \eqref{eq:thm:ident:unequal} holds.
\end{enumerate}
Then the order $\pi$ is identifiable.%
\end{thm}
In this condition not only do we need to control the descendants of a node, but also the other non-descendants that have not been identified. 

Before proving this result, we illustrate it with an example.
\begin{ex}
\label{ex:misspecification}
Let's consider a very simple case: A Markov chain with three nodes $X_1\rightarrow X_2 \rightarrow X_3$ such that
\begin{align*}
    \begin{split}
    & X_1 = \err_1 \sim N(0,1) \\
    & X_2 = \frac{1}{2}X_1^2 + \err_2,\ \ \ \ \err_2 \sim N(0,\tfrac{2}{3}) \\
    & X_3 = \frac{1}{3}X_2^2 + \err_3,\ \ \ \ \err_3 \sim N(0,\tfrac{1}{2}).
    \end{split}
\end{align*}
Here we have unequal residual variances. We now check that this model satisfies the conditions in Theorem~\ref{thm:ident:unequal}. Let $f(u)=u^{2}$ and note that the true ordering is $X_1\prec X_2 \prec X_3$. Starting from the first source node $X_1$, we have
\begin{align*}
    \begin{split}
        \sigma^2_1 
        &= 1\\
        \sigma_2^2 + \E\var(f(X_{\pi(2)}))
        &=2/3 + \var(X_1^2/2) = 2/3+1/2>1=\sigma^2_1 \\
        \sigma^2_3 + \E\var(f(X_{\pi(3)}))
        &=1/2 + \var(X_2^2/3) \\
        &= 1/2 + \frac{1}{9}\Big(\var(X_1^4)/16+\var(\err_2^2)+\var(X_1^2\err_2)\\
        & + 2\cov(X_1^4/4,X_1^2\err_2)+2\cov(\err_2^2,X_1^2\err_2)\Big)\\
        & = 1/2 + 8/9 + 8/81 > 1=\sigma^2_1 
    \end{split}
\end{align*}
Then for the second source node $X_2$,
\begin{align*}
    \begin{split}
    \sigma^2_2 & = 2/3 \\ 
    \sigma^2_3 + \E\var(f(X_{\pi(3)})\given X_1)&=1/2 + \frac{1}{9}(\var(\err^2_2)+\E\var(X^2_1\err_2\given X_1))\\
    & = 1/2 + 1/3 - 1/81 > 2/3
    \end{split}
\end{align*}
Thus the condition is satisfied. 

If instead we have $\sigma^2_3 = \var(\err_{3}) = 1/3$, the condition would be violated. It is easy to check that nothing changes for $X_1$.
For the second source node $X_2$, things are different:
\begin{align*}
    \sigma^2_3 + \E\var(f(X_{\pi(3)})\given X_1)=1/3+1/3-1/81<2/3
\end{align*}
Thus, the order of $X_2$ and $X_3$ would be flipped for this model. 

This example is easily confirmed in practice. We let $n$ range from 50 to 1000, and check if the estimated order is correct for the two models ($\sigma_3^2=1/2$ and $\sigma_3^2=1/3$). We simulated this 50 times and report the averages in Figure~\ref{fig:misspecification}.

\begin{figure}[t]
\centering
\includegraphics[width=.8\textwidth]{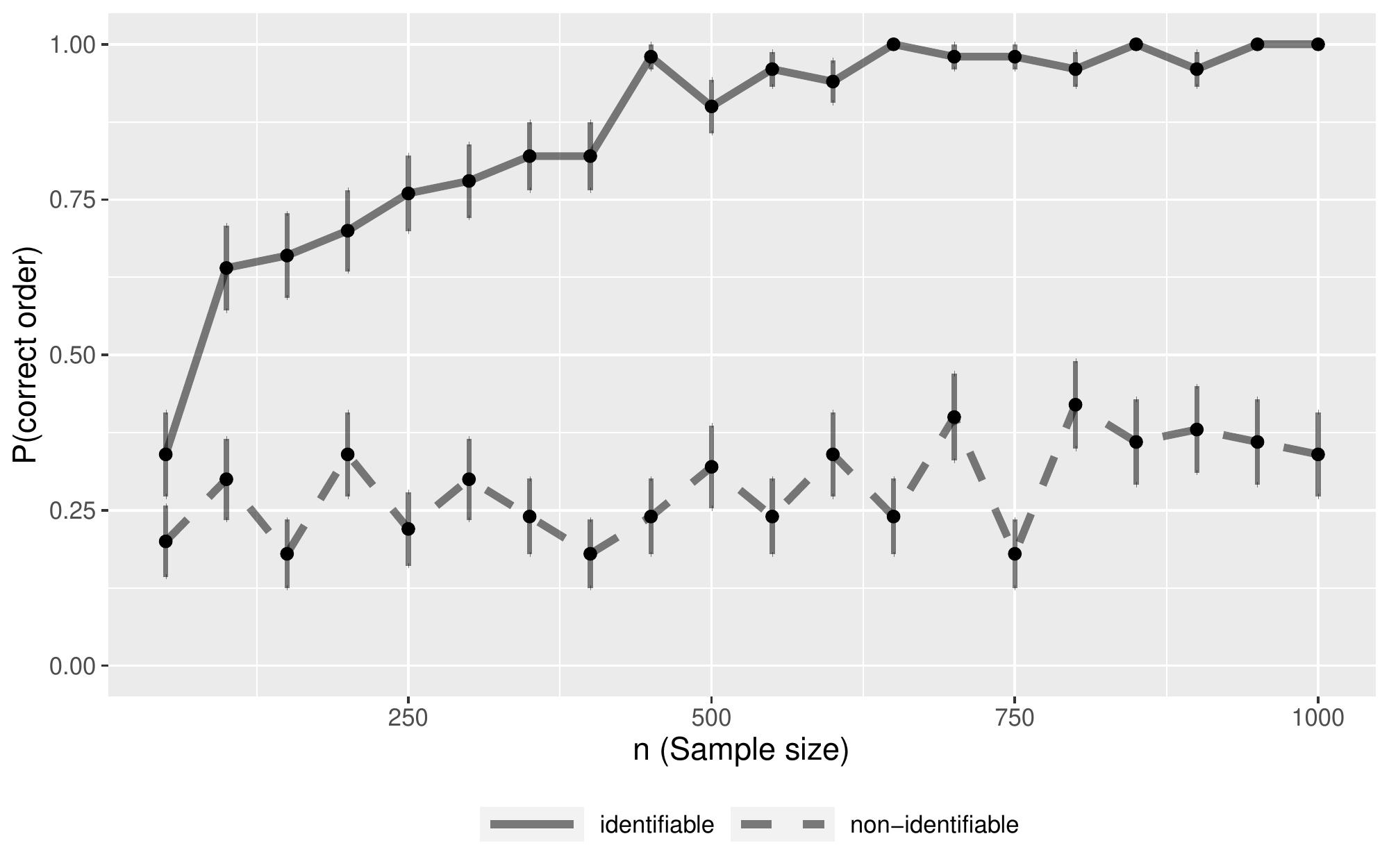}
\caption{Experiments confirming Example~\ref{ex:misspecification}. For the identifiable setting with $\sigma_3^2=1/2$, Algorithm~\ref{alg:eqvaranm:emp} correctly learns the topological ordering. For the non-identifiable setting with $\sigma_3^2=1/3$, Algorithm~\ref{alg:eqvaranm:emp} fails to learn the ordering.}
\label{fig:misspecification}
\end{figure}
\end{ex}

\begin{proof}[Proof of Theorem~\ref{thm:ident:unequal}]
We first consider the case where every node in the same layer has a different residual variance, i.e. $X_u,X_v\in \layer_m\implies\sigma^2_u\ne \sigma^2_v$. We proceed with induction on the element $j$ of the ordering $\pi$. Let $i=\pi_j$,

When $j=1$ and $i=\pi_{1}$,
$X_{i}$ must be a source node and we have for all $k\in \pi_{[2:d]}$, 
\begin{align*}
    \var(X_i)=\sigma_i^2< & \sigma_k^2+\E\var(\E(X_k\given \pa(k))) \\
    =& \E\var(X_k\given \pa(k)) + \var(\E(X_k\given \pa(k)))\\
    =&\var(X_k).
\end{align*}
Thus the first node to be identified must be $i=\pi_1$, as desired.

Now suppose the the first $j-1$ nodes in the ordering $\pi$ are correctly identified. 
The parent of node $i=\pi_j$ must have been identified in $\pi_{[j-1]}$ or it is a source node. Then we have for all $k\in \{\pi_{j+1},\ldots,\pi_d\}$, 
\begin{align*}
    \E\var(X_i\given X_{\pi_{[1:j-1]}})=\sigma_i^2<&\sigma_k^2+\E\var(\E(X_k\given \pa(k))\given X_{\pi_{[1:j-1]}}) \\
    =& \E\var(X_k\given \pa(k)) + \E\var(\E(X_k\given \pa(k))\given X_{\pi_{[1:j-1]}}) \\
    =& \E[\E(\var(X_k\given \pa(k))\given X_{\pi_{[1:j-1]}})] + \E[\var(\E(X_k\given \pa(k))\given X_{\pi_{[1:j-1]}})] \\
    =&\E\var(X_k\given X_{\pi_{[1:j-1]}})
\end{align*}
Then the $j^{\text{th}}$ node to be identified must be $i=\pi_j$. The induction is completed.

Finally, if $X_u,X_v\in \layer_{m}$ are in the same layer and $\sigma_u^2 = \sigma_v^2$, then this procedure may choose either $X_u$ or $X_v$ first. For example, if $X_u$ is chosen first, then $X_v$ will be incorporated into the same layer as $X_u$. Since both these nodes are in the same layer, swapping $X_{u}$ and $X_{v}$ in any ordering still produces a valid ordering of the DAG. The proof is complete.
\end{proof}

\section{Proof of Theorem~\ref{thm:main:sample}}
\label{app:proof:main}

The proof of Theorem~\ref{thm:main:sample} will be broken down into several steps. First, we derive an upper bound on the error of the plug-in estimator used in Algorithm~\ref{alg:eqvaranm:emp} (Appendix~\ref{app:proof:main:plugin}), and then we derive a uniform upper bound (Appendix~\ref{app:proof:main:uniform}). Based on this upper bound, we prove Theorem~\ref{thm:main:sample} via Proposition~\ref{prop:sample:bound} (Appendix~\ref{app:proof:main:thm}). Appendix~\ref{app:proof:main:technical} collects various technical lemmas that are used throughout.

\subsection{A plug-in estimate}
\label{app:proof:main:plugin}

Let $(X,Y)\in\R^{m}\times\R$ be a pair of random variables and $\wh{f}$ be a data-dependent estimator of the conditional expectation $\E[Y\given X]:=f(X)$. Assume we have split the sample into two groups, which we denote by $(U^{(1)},V^{(1)}),\ldots,(U^{(n_{1})},V^{(n_{1})})\sim \pr(X,Y)$ and $(X^{(1)},Y^{(1)}),\ldots,(X^{(n_{2})},Y^{(n_{2})})\sim \pr(X,Y)$ for clarity. Given these samples, define an estimator of $\sigma^{2}_{\resvar}:=\E\var(Y\given X)$ by
\begin{align}
\label{eq:def:plugin}
\wh{\sigma}_{\resvar}^{2}
:= \frac1{n_{2}}\sum_{i=1}^{n_{2}}(Y^{(i)})^{2} - \frac1{n_{2}}\sum_{i=1}^{n_{2}}\wh{f}(X^{(i)})^{2}.
\end{align}
Note here that $\wh{f}$ depends on $(U^{(i)},V^{(i)})$, and is independent of the second sample $(X^{(i)},Y^{(i)})$. 
We wish to bound the deviation $\pr(|\wh{\sigma}^{2}_{\resvar}-\sigma^{2}_{\resvar}| \ge t)$. 

Define the target $\theta^{*}=\E f^{2}(X)$ and its plug-in estimator
\begin{align*}
\theta(g; q)
= \int g(x)^{2}\dif q(x).
\end{align*}
Letting $\pr_{X}$ denote the true marginal distribution with respect to $x$, we have $\theta^{*}=\theta(f;\pr_{X})$ and $\wh{\theta}:=\theta(\wh{f};\wh{\pr})$, where $\wh{\pr}=n_{2}^{-1}\sum_{i}\delta_{X^{(i)}}$ is the empirical distribution.
We will also make use of more general targets $\theta(g;\pr_{X})=\E g^{2}(X)$ for general functions $g$.

Finally, as a matter of notation, we adopt the following convention: For a random variable $Z$, $\norm{Z}_{p}:=(\E_{Z} |Z|^{p})^{1/p}$ is the usual $L^{p}$-norm of $Z$ as a random variable, and for a (nonrandom) function $f$, $\norm{f}_{p}:=(\int|f|^{p}\,\dif\zeta)^{1/p}$, where $\zeta$ is a fixed base measure such as Lebesgue measure. In particular, $\norm{f(X)}_{p}\ne\norm{f}_{p}$. The difference of course lies in which measure integrals are taken with respect to. Moreover, we shall always explicitly specify with respect to which variables probabilities and expectations are taken, e.g. to disambiguate $\E_{\wh{f},X}=\E_{\wh{f}}\E_{X}$, $\E_{\wh{f}}$, and $\E_{X}$.

We first prove the following result:
\begin{prop}
\label{prop:plugin:upper:L2}
Assume $\norm{f(X)}_{\infty},\norm{\wh{f}(X)}_{\infty}\le B_{\infty}$.
Then 
\begin{align}
\label{eq:prop:plugin:upper:bound:L2}
\pr(|\wh{\theta}-\theta^{*}| \ge 2t) 
\les \frac{\E_{\wh{f}}\norm{f(X)-\wh{f}(X)}_{2}^{2}}{t^{2}} + \frac{\norm{f(X)}_{4}^{4} + \E_{\wh{f}} \norm{\wh{f}(X)-f(X)}_{4}}{n_{2}t^{2}}.
\end{align}
\end{prop}

\begin{proof}

We have
\begin{align*}
\pr_{\wh{f},X}(|\wh{\theta}-\theta^{*}| \ge 2t)
\le \pr_{\wh{f},X}(|\theta^{*}-\theta(\wh{f};\pr_{X})| \ge t) + \pr_{\wh{f},X}(|\wh{\theta}-\theta(\wh{f};\pr_{X})| \ge t).
\end{align*}
The second term is easy to dispense with since
\begin{align}
\label{eq:prop:plugin:upper:2}
\pr_{\wh{f},X}(|\wh{\theta}-\theta(\wh{f};\pr_{X})| \ge t)
&\le \frac{\E_{\wh{f}}\E_{X}(\wh{\theta}-\theta(\wh{f};\pr_{X}))^{2}}{t^{2}}
\le \frac{\E_{\wh{f}}\var_{X}(\wh{f}^{2}(X))}{n_{2}t^{2}}.
\end{align}
It follows from Lemma~\ref{lem:lp:mom:gen} with $p=4$ that
\begin{align*}
|\E_{X} \wh{f}(X)^{4}-\E_{X} f(X)^{4}| 
&\le \norm{\wh{f}(X)-f(X)}_{4}\sum_{k=0}^{3}\norm{\wh{f}(X)}_{4}^{k}\norm{f(X)}_{4}^{3-k}\\
&\le 4(2B_{\infty})^{3}\norm{\wh{f}(X)-f(X)}_{4}
\end{align*}
and hence
\begin{align}
\label{eq:prop:plugin:upper:5}
\var_{X}(\wh{f}^{2}(X))
\le \E_{X}\wh{f}^{4}(X)
\les \norm{f(X)}_{4}^{4} + \norm{\wh{f}(X)-f(X)}_{4}.
\end{align}
Combined with \eqref{eq:prop:plugin:upper:2}, we finally have 
\begin{align}
\label{eq:prop:plugin:upper:6}
\pr(|\wh{\theta}-\theta(\wh{f};\pr_{X})| \ge t)
\les \frac{\norm{f(X)}_{4}^{4} + \E_{\wh{f}}\norm{\wh{f}(X)-f(X)}_{4}}{n_{2}t^{2}}.
\end{align}

For the first term, since $f(X),\wh{f}(X)\in L^{\infty}$, Lemma~\ref{lem:l4:ineq} implies
\begin{align*}
\E_{\wh{f}}(\theta^{*}-\theta(\wh{f};\pr_{X}))^{2}
&= \E_{\wh{f}}\E_{X}(f^{2}(X)-\wh{f}^{2}(X))^{2} \\
&\les \E_{\wh{f}}\norm{f(X)-\wh{f}(X)}_{2}^{2},
\end{align*}
and thus
\begin{align*}
\pr_{\wh{f}}(|\theta^{*}-\theta(\wh{f};\pr_{X})| \ge t) 
&\le \frac{\E_{\wh{f}}(\theta^{*}-\theta(\wh{f};\pr_{X}))^{2}}{t^{2}}
\les \frac{\E_{\wh{f}}\norm{f(X)-\wh{f}(X)}_{2}^{2}}{t^{2}}.
\end{align*}
Therefore
\begin{align*}
\pr_{\wh{f},X}(|\wh{\theta}-\theta^{*}| \ge 2t)
&\le \pr_{\wh{f},X}(|\theta^{*}-\theta(\wh{f};\pr_{X})| \ge t) + \pr_{\wh{f},X}(|\wh{\theta}-\theta(\wh{f};\pr_{X})| \ge t) \\
&\les \frac{\E_{\wh{f}}\norm{f(X)-\wh{f}(X)}_{2}^{2}}{t^{2}} + \frac{\norm{f(X)}_{4}^{4} + \E_{\wh{f}} \norm{\wh{f}(X)-f(X)}_{4}}{n_{2}t^{2}}.
\qedhere
\end{align*}
\end{proof}

Finally, we conclude the following:
\begin{cor}
\label{prop:ecv:estprob}
If $\norm{f(X)}_{\infty},\norm{\wh{f}(X)}_{\infty}\le B_{\infty}$, then
\begin{align*}
\pr(|\wh{\sigma}_{\resvar}^{2} - \sigma_{\resvar}^{2}| \ge t)
&\les \frac{4}{t^{2}}\Big(\E_{\wh{f}}\norm{f(X)-\wh{f}(X)}_{2}^{2} + \frac{\var(Y) + \norm{f(X)}_{4}^{4} + \E_{\wh{f}} \norm{\wh{f}(X)-f(X)}_{4}}{n_{2}}\Big).
\end{align*}
\end{cor}

\subsection{A uniform bound}
\label{app:proof:main:uniform}

For any $j=1,\ldots,r$ and $\ell\notin \anc_{j}$, define $\sigma_{\ell j}^{2}:=\E\var(X_{\ell}\given \anc_{j})$ and $\wh{\sigma}_{\ell j}^{2}$ the corresponding plug-in estimator from \eqref{eq:def:plugin}.
By Proposition~\ref{prop:ecv:estprob}, we have for $f_{\ell j}(X_{\anc_{j}}):=\E[X_{\ell}\given X_{\anc_{j}}]$,
\begin{align}
\label{eq:ecv:estprob}
\begin{aligned}
\pr(|\wh{\sigma}_{\ell j}^{2} - \sigma_{\ell j}^{2}| \ge t)
\les \frac{4}{t^{2}}\Big( 
\E_{\wh{f}}\norm{f_{\ell j}(X_{\anc_{j}}) &- \wh{f}_{\ell j}(X_{\anc_{j}})}_{2}^{2}  \\
&+ \frac{\var(X_{\ell}) + \norm{f_{\ell j}(X_{\anc_{j}})}_{4}^{4} + \E_{\wh{f}} \norm{\wh{f}_{\ell j}(X_{\anc_{j}})-f_{\ell j}(X_{\anc_{j}})}_{4}}{n_{2}}
\Big)
\end{aligned}
\end{align}
Thus we have the following result: 
\begin{prop}
\label{prop:ecv:estprob:uniform}
Assume for all $j$ and $\ell\notin \anc_{j}$:
\begin{enumerate}
\item $\norm{f_{\ell j}(X_{\anc_{j}})}_{\infty},\norm{\wh{f}_{\ell j}(X_{\anc_{j}})}_{\infty}\le B_{\infty}$;
\item $\norm{f_{\ell j}(X_{\anc_{j}})}_{4}^{4}+\var(X_{\ell})\le B_{\infty}$;
\item $\E_{\wh{f}}\norm{f_{\ell j}(X_{\anc_{j}})-\wh{f}_{\ell j}(X_{\anc_{j}})}_{2}^{2}\to 0$.
\end{enumerate}
Then %
\begin{align}
\label{eq:prop:ecv:estprob:uniform}
\sup_{\ell, j}\pr(|\wh{\sigma}_{\ell j}^{2} - \sigma_{\ell j}^{2}| > t)
\les \frac{4}{t^{2}}\Big(\E_{\wh{f}}\norm{f_{\ell j}(X_{\anc_{j}})-\wh{f}_{\ell j}(X_{\anc_{j}})}_{2}^{2} + \frac{1}{n_{2}}\Big).
\end{align}
\end{prop}
For example, under Conditions~\ref{cond:reg}-\ref{cond:est}, we have $B_{\infty}:=\sup_{\ell,j}2\norm{f_{\ell j}(X_{\anc_{j}})}_{\infty}^{4} + \zeta_{0}$.

\subsection{Proof of Theorem~\ref{thm:main:sample}}
\label{app:proof:main:thm}

Recall $\sigma_{\ell j}^{2}=\E\var(X_{\ell}\given \anc_{j})$ and $\wh{\sigma}_{\ell j}^{2}$ is the plug-in estimator defined by \eqref{eq:def:plugin}.
Let $\resvarbd>0$ be such that
\begin{align}
\label{eq:ecv:delta:prob}
\sup_{\ell,j}\pr(|\wh{\sigma}_{\ell j}^{2} - \sigma_{\ell j}^{2}| > t)
\le \frac{\resvarbd^{2}}{t^{2}}.
\end{align}
For example, Proposition~\ref{prop:ecv:estprob:uniform} implies $\resvarbd^{2}\asymp\delta^{2}+n_{2}^{-1}$. Recall also $\Delta:=\inf_{j}\Delta_{j}$, where $\Delta_{j}>0$ is the smallest number such that $\E\var(X_{\ell}\given \anc_{j})>\sigma^{2}+\Delta_{j}$ for all $\ell\notin \anc_{j}$.

Theorem~\ref{thm:main:sample} follows immediately from Proposition~\ref{prop:sample:bound} below, combined with Proposition~\ref{prop:ecv:estprob:uniform} to bound $\resvarbd$ by $\npestbd$.

\begin{prop}
\label{prop:sample:bound}
Define $\resvarbd>0$ as in \eqref{eq:ecv:delta:prob}. Then for any threshold $\resvarbd\sqrt{d}<\eta<\Delta/2$, we have
\begin{align*}
\pr(\wh{\layer} = \layer(\gr))
&\ge 1 - \frac{\resvarbd^{2}}{\eta^{2}}rd.
\end{align*}
\end{prop}

\begin{proof}
Define $\mathcal{E}_{j-1}:=\{\wh{\layer}_{1}=\layer_{1},\ldots,\wh{\layer}_{j-1}=\layer_{j-1}\}$. It follows that
\begin{align*}
\pr(\wh{\layer} = \layer(\gr))
&= \pr(\wh{\layer}_{1}=\layer_{1},\ldots,\wh{\layer}_{r}=\layer_{r}) 
= \prod_{j=1}^{r}\pr(\wh{\layer}_{j}=\layer_{j}\given \mathcal{E}_{j-1}).
\end{align*}
Clearly, if $\wh{\layer}_{1}=\layer_{1},\ldots,\wh{\layer}_{r}=\layer_{r}$ then $\wh{r}=r$.
By definition, we have $\sigma_{\ell j}^{2}>\sigma^{2}+\Delta$, i.e. $\Delta$ is the smallest ``gap'' between any source in a subgraph $\gr[V-\anc_{j-1}]$ and the rest of the nodes.

Now let $\wh{\sigma}^{2}:=\min_{\ell}\wh{\sigma}_{\ell 0}^{2}$ and consider $\layer_{1}$: 
\begin{align*}
\pr(\wh{\layer}_{1}=\layer_{1})
= \pr(|\wh{\sigma}_{\ell 0}^{2} - \wh{\sigma}^{2}|\le\eta\,\forall\ell\in \anc_{1},\,|\wh{\sigma}_{\ell 0}^{2} - \wh{\sigma}^{2}|>\eta\,\forall\ell\notin \anc_{1})
\end{align*}
Now for any $k,\ell\in \layer_{j}$,
\begin{align*}
|\wh{\sigma}_{\ell j}^{2} - \wh{\sigma}_{k j}^{2}|
&\le |\wh{\sigma}_{\ell j}^{2} - \sigma^{2}| + |\wh{\sigma}_{k j}^{2} - \sigma^{2}|,
\end{align*}
and for any $\ell\notin \layer_{j}$ and $k\in \layer_{j}$,
\begin{align*}
|\wh{\sigma}_{\ell j}^{2} - \wh{\sigma}_{k j}^{2}|
&> \Delta_{j} - |\sigma_{\ell j}^{2} - \wh{\sigma}_{\ell j}^{2}| - |\wh{\sigma}_{k j}^{2} - \sigma_{k j}^{2}|.
\end{align*}
Thus, with probability $1-d\resvarbd^{2}/t^{2}$, we have 
\begin{align*}
|\wh{\sigma}_{\ell j}^{2} - \wh{\sigma}_{k j}^{2}|
&\le 2t
\quad\text{if $k,\ell\in \layer_{j}$, and} \\
|\wh{\sigma}_{\ell j}^{2} - \wh{\sigma}_{k j}^{2}|
&> \Delta_{j}-2t
\quad\text{if $\ell\notin \layer_{j}$ and $k\in \layer_{j}$}.
\end{align*}
Now, as long as $t<\Delta/4$, we have $\Delta_{j}-2t > 2t$, which implies that $\eta:=2t<\Delta/2$. 

Finally, we have
\begin{align*}
\pr(\wh{\layer}_{1}\ne \layer_{1})
&\le \frac{\resvarbd^{2}}{\eta^{2}}d
\implies 
\pr(\wh{\layer}_{1}=\layer_{1})
\ge 1 - \frac{\resvarbd^{2}}{\eta^{2}}d.
\end{align*}
Recall that $d_{j}:=|\layer_{j}|$. Then by a similar argument
\begin{align*}
\pr(\wh{\layer}_{2} = \layer_{2}\given \wh{\layer}_{1}= \layer_{1})
&\ge 1 - \frac{\resvarbd^{2}}{\eta^{2}}(d-d_{1}).
\end{align*}
Recalling $\mathcal{E}_{j-1}:=\{\wh{\layer}_{1}=\layer_{1},\ldots,\wh{\layer}_{j-1}=\layer_{j-1}\}$, we have just proved that $\pr(\wh{\layer}_{2} = \layer_{2}\given \mathcal{E}_{1})\ge 1-(d-d_{1})(\resvarbd^{2}/\eta^{2})$. 
A similar argument proves that $\pr(\wh{\layer}_{j} = \layer_{j}\given \mathcal{E}_{j-1})\ge 1-(\resvarbd^{2}/\eta^{2})(d-d_{j-1})$. 
Since $\eta>\resvarbd\sqrt{d}$,
the inequality $\prod_{j}(1-x_{j})\ge 1-\sum_{j}x_{j}$ implies
\begin{align*}
\pr(\wh{\layer} = \layer(\gr))
&= \prod_{j=1}^{r}\pr(\wh{\layer}_{j}=\layer_{j}\given\mathcal{E}_{j-1}) \\
&= \prod_{j=1}^{r}\Big(1 - \frac{\resvarbd^{2}}{\eta^{2}}(d-d_{j-1})\Big) \\
&\ge 1- \sum_{j=1}^{r}\frac{\resvarbd^{2}}{\eta^{2}}(d-d_{j-1}) \\
&\ge 1- \frac{\resvarbd^{2}}{\eta^{2}}rd
\end{align*}
as desired.
\end{proof}

\subsection{Technical lemmas}
\label{app:proof:main:technical}

\begin{lemma}
\label{lem:lp:mom:gen}
\begin{align*}
|\E X^{p}-\E Y^{p}| 
\le \norm{X-Y}_{p}\sum_{k=0}^{p-1}\norm{X}_{p}^{k}\norm{Y}_{p}^{p-1-k}
\end{align*}
\end{lemma}
\begin{proof}
Write $\E X^{p} - \E Y^{p}$ as a telescoping sum:
\begin{align*}
|\E X^{p} - \E Y^{p}|
= |\norm{X}_{p}^{p} - \norm{Y}_{p}^{p}|
&= |\norm{X}_{p}-\norm{Y}_{p}|\cdot\sum_{k=0}^{p-1}\norm{X}_{p}^{k}\norm{Y}_{p}^{p-k-1} \\
&\le \norm{X-Y}_{p}\cdot\sum_{k=0}^{p-1}\norm{X}_{p}^{k}\norm{Y}_{p}^{p-k-1}.
\qedhere
\end{align*}
\end{proof}

\begin{lemma}
\label{lem:lp:interp}
Fix $p>2$ and $\delta > 0$ and suppose $\norm{f}_{p+\delta},\norm{g}_{p+\delta}\le B_{p+\delta}$. Then
\begin{align*}
\norm{f-g}_{p}
\le C_{p,\delta}\cdot\norm{f-g}_{2}^{\gamma_{p,\delta}},
\quad 
C_{p,\delta} 
= (2B_{p+\delta})^{\tfrac{(p-2)(p+\delta)}{p(p+\delta-2)}},
\quad
\gamma_{p,\delta}
= \frac{2\delta}{p(p+\delta-2)}.
\end{align*}
The exponent $\gamma_{p,\delta}$ satisfies $\gamma_{p,\delta}\le 2/p<1$ and $\gamma_{p,\delta}\to2/p$ as $\delta\to\infty$, and the constant $C_{p,\delta}\to 1-\tfrac{2}{p}$ as $\delta\to\infty$. Thus, if $f-g\in L^{\infty}$, then
\begin{align*}
\norm{f-g}_{p}
\les \norm{f-g}_{2}^{2/p}.
\end{align*}
\end{lemma}

\begin{proof}
Use log-convexity of $L^{p}$-norms with $2=q<p<r=p+\delta$.
\end{proof}

\begin{lemma}
\label{lem:l4:ineq}
Assume $\norm{f}_{\infty}\le B_{\infty}<\infty$ and $g\in L^4$. Then
\begin{align}
\begin{aligned}
\E_{X}(f(X)^{2}-g(X)^{2})^{2} 
\le \norm{f(X)&-g(X))}_{4}^{4} +\\
&4B_{\infty}\norm{f(X)-g(X)}_{3}^{3} + 4B_{\infty}^{2}\norm{f(X)-g(X)}_{2}^{2}.
\end{aligned}
\end{align}
If additionally $g(X)\in L^{\infty}$, then 
\begin{align}
\E_{X}(f(X)^{2}-g(X)^{2})^{2} 
&\les \norm{f(X)-g(X)}_{2}^{2}.
\end{align}
\end{lemma}
\begin{proof}
Note that
\begin{align*}
(f(X)^{2}-g(X)^{2})^{2} 
&= (g(X)-f(X))^{4} + 4f(X)(g(X)-f(X))^{3} + 4f(X)^{2}(g(X)-f(X))^{2} \\
&\le (g(X)-f(X))^{4} + 4|f(X)||g(X)-f(X)|^{3} + 4f(X)^{2}(g(X)-f(X))^{2}.
\end{align*}
Thus
\begin{align*}
\E(f(X)^{2}-g(X)^{2})^{2} 
&\le \E(g(X)-f(X))^{4} +\\
&\qquad\qquad 4\E\big[|f(X)||g(X)-f(X)|^{3}\big] + 4\E\big[f(X)^{2}(g(X)-f(X))^{2}\big] \\
&\le \E(g(X)-f(X))^{4} + \\
&\qquad\qquad 4B_{\infty}\E|g(X)-f(X)|^{3} + 4B_{\infty}^{2}\E(g(X)-f(X))^{2}.
\end{align*}
This proves the first inequality. The second follows from taking $\delta\to\infty$ in Lemma~\ref{lem:lp:interp} and using $\norm{f-g}_{p}^{p}\les \norm{f-g}_{2}^{2}$ for $p=3,4$.
\end{proof}

\section{Comparison to CAM algorithm}
\label{app:cam}

In this section, we justify the claim in Example~\ref{ex:cam:fail} that \emph{there exist infinitely many nonlinear functions $g$ for which the CAM algorithm returns an incorrect graph under the model \eqref{eq:cam:fail}.} To show this, we first construct a \emph{linear}  model on which CAM returns an incorrect ordering. Since CAM focuses on nonlinear models, we then show that this extends to any sufficiently small nonlinear perturbation of this model.

The linear model is 
\begin{align}
\label{eq:lin:cam:fail}
\left\{
\begin{aligned}
    &X_1\sim \normalN(0,1) \\
    &X_2 = X_1 + \err_{2} \ \ \ \ \err_{2}\sim \normalN(0,1)\\
    &X_3 = X_1 + X_2 + \err_{3} \ \ \ \ \err_{3}\sim \normalN(0,1).
\end{aligned}
\right.
\end{align}
The graph is 
\begin{align*}
    \begin{split}
    X_1\rightarrow & X_2\\
    \searrow&\downarrow\\
    & X_3
    \end{split}
\end{align*}
which corresponds to the adjacency matrix
\begin{align*}
    \begin{pmatrix}
    0 & 1 & 1 \\
    0 & 0 & 1 \\
    0 & 0 & 0  \\
    \end{pmatrix}.
\end{align*}
The \emph{IncEdge} step of the CAM algorithm (\sec5.2, \citep{buhlmann2014}) is based on the following score function:
\begin{align*}
\sum_{j=1}^{d}\log\Big(
    \E\var(X_{j}\given \pa(j))
\Big).
\end{align*}
The algorithm starts with the empty DAG (i.e. $\pa(j)=\emptyset$ for all $j$) and proceeds by greedily adding edges that decrease this score the most in each step. For example, in the first step, CAM searches for the pair $(X_{i},X_{j})$ that maximizes $\log\var(X_{j}) - \log\E\var(X_{j}\given X_{i})$, and adds the edge $X_{i}\to X_{j}$ to the estimated DAG. The second proceeds similarly until the estimated order is determined. Thus, it suffices to study the log-differences $\diff(j,i,S):=\log\var(X_{j}\given X_{S}) - \log\E\var(X_{j}\given X_{S\cup i})$.

The following are straightforward to compute for the model \eqref{eq:lin:cam:fail}:
\begin{align*}
    \begin{split}
    & \var(X_1) = 1 \ \ \ \ \E\var(X_2|X_1)=1           \ \ \ \ \E\var(X_1|X_2)=\frac{1}{2}\\
    & \var(X_2) = 2 \ \ \ \ \E\var(X_3|X_1)=2           \ \ \ \ \E\var(X_1|X_3)=\frac{3}{2}\\
    & \var(X_3) = 6 \ \ \ \ \E\var(X_3|X_2)=\frac{1}{3} \ \ \ \ \!\E\var(X_2|X_3)=\frac{1}{2}\\
    \end{split}
\end{align*}
Then
\begin{align*}
    \begin{split}
    \log\Big(\frac{\var(X_2)}{\E\var(X_2|X_1)}\Big)
    =\log\Big(\frac{2}{1}\Big)
    =\log2 \ \ \  \ 
    &\log\Big(\frac{\var(X_1)}{\E\var(X_1|X_2)}\Big)
    =\log\Big(\frac{1}{1/2}\Big)
    =\log2\\
    \log\Big(\frac{\var(X_3)}{\E\var(X_3|X_1)}\Big)
    =\log\Big(\frac{6}{2}\Big)
    =\log3 \ \ \  \ 
    &\log\Big(\frac{\var(X_1)}{\E\var(X_1|X_3)}\Big)
    =\log\Big(\frac{1}{1/3}\Big)
    =\log3\\
    \log\Big(\frac{\var(X_3)}{\E\var(X_3|X_2)}\Big)
    =\log\Big(\frac{6}{3/2}\Big)
    =\log4\ \ \  \ 
    &\log\Big(\frac{\var(X_2)}{\E\var(X_2|X_3)}\Big)
    =\log\Big(\frac{2}{1/2}\Big)
    =\log4\\
    \end{split} 
\end{align*}
Now, if $X_{3}\to X_{2}$ is chosen first, then the order is incorrect and we are done. Thus suppose CAM instead chooses $X_{2}\to X_{3}$, then in the next step it would update the score for $X_{1}\to X_{3}$ to be 
\begin{align*}
    \log\Big(\frac{\E\var(X_3|X_2)}{\E\var(X_3|X_1,X_2)}\Big)
    =\log\Big(\frac{3/2}{1}\Big)
    =\log\frac{3}{2}
    <\log\Big(\frac{\E\var(X_1|X_2)}{\E\var(X_1|X_3,X_2)}\Big)
    =\log3
\end{align*}
Therefore, for the next edge, CAM would choose $X_{3}\to X_{1}$, which also leads to the wrong order. Thus regardless of which edge is selected first, CAM will return the wrong order.

Thus, when CAM is applied to data generated from \eqref{eq:lin:cam:fail}, it is guaranteed to return an incorrect ordering. Although the model \eqref{eq:lin:cam:fail} is identifiable, it does not satisfy the identifiability condition for CAM (Lemma~1, \citep{buhlmann2014}), namely that the structural equation model is a nonlinear additive model. Thus, we need to extend this example to an identifiable, nonlinear additive model.

Since this depends only on the scores $\diff(j,i,S)$, it suffices to construct a \emph{nonlinear} model with similar scores. For this, we consider a simple nonlinear extension of \eqref{eq:lin:cam:fail}: Let $g$ be an arbitrary bounded, nonlinear function, and define $g_{\delta}(u):=u+\delta g(u)$. The nonlinear model is given by
\begin{align}
\label{eq:nlin:cam:fail}
\left\{
\begin{aligned}
    &X_1\sim \normalN(0,1) \\
    &X_2 = g_{\delta}(X_1) + \err_{2} \ \ \ \ \err_{2}\sim \normalN(0,1)\\
    &X_3 = g_{\delta}(X_1) + g_{\delta}(X_2) + \err_{3} \ \ \ \ \err_{3}\sim \normalN(0,1).
\end{aligned}
\right.
\end{align}
This model satisfies both our identifiability condition (Condition~\ref{cond:ident}) and the identifiability condition for CAM (Lemma~1, \citep{buhlmann2014}).

We claim that for sufficiently small $\delta$, the CAM algorithm will return the wrong ordering (see Proposition~\ref{prop:cam:nonlinear} below for a formal statement). It follows that the scores $\diff(j,i,S;\delta)$ corresponding to the model \eqref{eq:nlin:cam:fail} can be made arbitrarily close to $\diff(j,i,S)=\diff(j,i,S;0)$, which implies that CAM will return the wrong ordering for sufficiently small $\delta>0$.

\begin{figure}
\centering
\includegraphics[width=\textwidth]{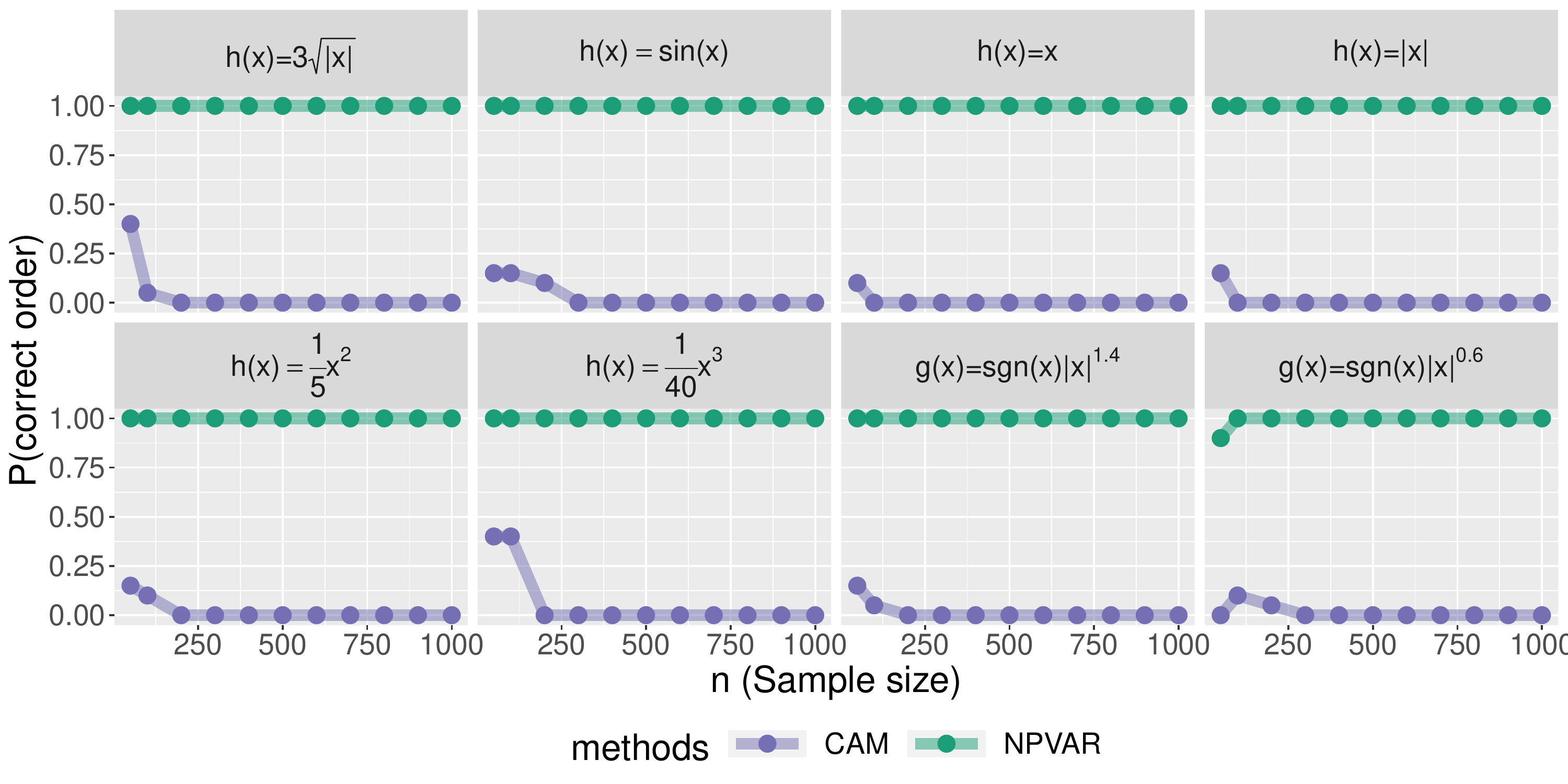}\\
\caption{The CAM algorithm does not recover the correct ordering under different nonlinear functions and models. $h(x)$ refers to model (\ref{eq:nlin2:cam:fail}), $g(x)$ refers to model (\ref{eq:cam:fail}) respectively.}
\label{fig:camfail:full}
\end{figure}

In Figure~\ref{fig:camfail:full}, we illustrate this empirically. In addition to the model \eqref{eq:nlin:cam:fail}, we also simulated from the following model, which shows that this phenomenon is not peculiar to the construction above:
\begin{align}
\label{eq:nlin2:cam:fail}
\left\{
\begin{aligned}
    &X_1\sim \normalN(0,1) \\
    &X_2 = X_1^{2} + \err_{2} \ \ \ \ \err_{2}\sim \normalN(0,1)\\
    &X_3 = 4X_1^2 + h(X_2) + \err_{3} \ \ \ \ \err_{3}\sim \normalN(0,1).
\end{aligned}
\right.
\end{align}
In all eight examples, NPVAR perfectly recovers the ordering while CAM is guaranteed to return an inconsistent order for sufficiently large $n$ (i.e. once the scores are consistently estimated).

\begin{prop}
\label{prop:cam:nonlinear}
Let $\E_{\delta}$ and $\var_{\delta}$ be taken with respect to model \eqref{eq:nlin:cam:fail}.
Then for all $i,j\in\{1,2,3\}$, as $\delta\to0$,
\begin{align*}
\begin{split}
    |\var_{\delta}(X_i)-\var_{0}(X_i)| &= o(1), \\
    |\E_{\delta}\var_{\delta}(X_i|X_j)-\E_{0}\var_{0}(X_i|X_j)| &= o(1).
\end{split}
\end{align*} 
and $|\E_{\delta}\var_{\delta}(X_3|X_1,X_2)-\E_{0}\var_{0}(X_3|X_1,X_2)| = o(1)$.
\end{prop}

\begin{proof}[Proof sketch of Proposition~\ref{prop:cam:nonlinear}]
The proof is consequence of the fact that the differences are continuous functions of $\delta$. We sketch the proof for $i=2$; the remaining cases are similar.

We have
\[
\begin{split}
    & \var_{\delta}(X_2)=\var_{\delta}(g(X_1))+\var_{\delta}(\epsilon) \\
    & \var_{0}(X_2)=\var_{0}(X_1)+\var_{0}(\epsilon) \\
\end{split}
\]
Let $\phi(t)$ be the standard normal density. We only need to analyze and compare $\var_{0}(X_1) - \var_{\delta}(g(X_1))=\var_{0}(X_1) - \var_{0}(g(X_1))$ in two parts:
\[
\begin{split}
    & \int (X_1^2-g(X_1)^2) \phi(X_1)dX_1 \\
    & (\int X_1 \phi(X_1)dX_1)^2- (\int g(X_1) \phi(X_1)dX_1)^2.
\end{split}
\]
Since $|X_1-g(X_1)|\le \delta$, 
\[
\begin{split}
    & |\int (X_1^2-g(X_1)^2) \phi(X_1)dX_1| \le  \delta\int|X_1+g(X_1)|\phi(X_1)dX_1\le \delta\int2|X_1|\phi(X_1)dX_1+\delta^2 = C\delta +\delta^2\\
    &|\int g(X_1) \phi(X_1)dX_1| = |\int( X_1-g(X_1)) \phi(X_1)dX_1| \le \delta \int\phi(X_1)dX_1= \delta\\
    & |\int (X_1+g(X_1)\phi(X_1))dX_1|=|\mathbb{E}g(X_1)|\le \delta.
\end{split}
\]     
Thus 
\[
    |(\int X_1 \phi(X_1)dX_1)^2- (\int g(X_1) \phi(X_1)dX_1)^2|\le \delta^2,
\]     
so that
\[
    |\var_{\delta}(X_2)-\var_{0}(X_2)|=o(1)
\]       
as claimed.
\end{proof}

\section{Experiment details}
\label{app:exp}

In this appendix we outline the details of our experiments, as well as additional simulations.

\subsection{Experiment settings}
\label{app:exp:settings}

For graphs, we used
\begin{itemize}
\item \textit{Markov chain (MC)}. Graph where there is one edge from $X_{i-1}$ to $X_i$ for all nodes $i=2,\ldots,d$.
\item \textit{Erd\H{o}s R\'{e}nyi (ER)}. Random graphs whose edges are added independently with specified expected number of edges. 
\item \textit{Scale-free networks (SF)}. Networks simulated according to the Barabasi-Albert model. 
\end{itemize}

For models, we refer to the nonlinear functions in SEM. We specify the nonlinear functions in $X_j=f_j(X_{\text{pa(j)}})+\err_j$ for all $j=1,2,\ldots,d$, where $\err_j\iid\normalN(0,\sigma^{2})$ with variance $\sigma^2\in\{0.2,0.5,0.8\}$
\begin{itemize}
\item \textit{Additive sine model (SIN)}: $f_j(X_{\text{pa}(j)})=\sum_{k\in \text{pa}(j)}f_{jk}(X_k)$ where $f_{jk}(X_k)=\sin(X_k)$.
\item \textit{Additive Gaussian process (AGP)}: $f_j(X_{\text{pa}(j)})=\sum_{k\in \text{pa}(j)}f_{jk}(X_k)$ where $f_{jk}$ is a draw from Gaussian process with RBF kernel with length-scale one. 
\item \textit{Non-additive Gaussian process (NGP)}: $f_j$ is a draw from Gaussian process with RBF kernel with length-scale one. 
\item \textit{Generalized Linear Model (GLM)}: This is a special case with non-additive model and non-additive noise. We specify the model $\mathbb{P}(X_j=1)=f_j(X_{\text{pa}(j)})$ by a parameter $p\in \{0.1,0.3\}$. Given $p$, if $j$ is a source node, $\mathbb{P}(X_j=1)=p$. For the Markov chain model we define:
\begin{align*}
f_j(X_{\text{pa}(j)})=\left\{\begin{aligned}
p& \ \ \ \ & X_k=1\\
1-p& \ \ \ \ & X_k=0\\
\end{aligned}\right. 
\end{align*} or vice versa.
\end{itemize}

We generated graphs from each of the above models with $\{d,4d\}$ edges each. These are denoted by the shorthand XX-YYY-$k$, where XX denotes the graph type, YYY denotes the model, and $k$ indicates the graphs have $kd$ edges on average.
For example, ER-SIN-1 indicates an ER graph with $d$ (expected) edges under the additive sine model. SF-NGP-4 indicates an SF graph with $4d$ (expected) edges under a non-additive Gaussian process. 
Note that there is no difference between additive or non-additive GP for Markov chains, so the results for MC-NGP-$k$ are omitted from the Figures in Appendix~\ref{app:exp:expts}.

Based on these models, we generated random datasets with $n$ samples. For each simulation run, we generated $n\in\{100,200,500,750,1000\}$ samples for graphs with $d \in \{5,10,20,40,50,60,70\}$ nodes. 

\subsection{Implementation and baselines}
\label{app:exp:baselines}

Code implementing \NPVAR{} can be found at \url{https://github.com/MingGao97/NPVAR}.
We implemented \NPVAR{} (Algorithm~\ref{alg:eqvaranm:emp}) with generalized additive models (GAMs) as the nonparametric estimator for $\wh{f}_{\ell j}$. We used the \texttt{gam} function in the R package \texttt{mgcv} with P-splines \texttt{bs='ps'} and the default smoothing parameter \texttt{sp=0.6}. In all our implementations, including our own for Algorithm~\ref{alg:eqvaranm:emp}, we used default parameters in order to avoid skewing the results in favour of any particular algorithm as a result of hyperparameter tuning.

We compared our method with following approaches as baselines:
\begin{itemize}
\item Regression with subsequent independence test (RESIT) identifies and disregards a sink node at each step via independence testing \citep{peters2014}. The implementation is available at \url{http://people.tuebingen.mpg.de/jpeters/onlineCodeANM.zip}. Uses HSIC test for independence testing with \texttt{alpha=0.05} and \texttt{gam} for nonparametric regression. 
\item Causal additive models (CAM) estimates the topological order by greedy search over edges after a preliminary neighborhood selection \citep{buhlmann2014}. The implementation is available at \url{https://cran.r-project.org/src/contrib/Archive/CAM/}. By default, \CAM{} applies an extra pre-processing step called preliminary neighborhood search (PNS). Uses \texttt{gam} to compute the scores and for pruning, and \texttt{mboost} for preliminary neighborhood search. The R implementation of CAM does not use the default parameters for \texttt{gam} or \texttt{mboost}, and instead optimizes these parameters at runtime.
\item NOTEARS uses an algebraic characterization of DAGs for score-based structure learning  of nonparametric models via partial derivatives \citep{zheng2018notears,zheng2019learning}. The implementation is available at \url{https://github.com/xunzheng/notears}. We used neural networks for the nonlinearities with a single hidden layer of 10 neurons. The training parameters are \texttt{lambda1=0.01}, \texttt{lambda2=0.01} and the threshold for adjacency matrix is \texttt{w\_threshold=0.3}.
\item Greedy equivalence search with generalized scores (GSGES) uses generalized scores for greedy search without assuming model class \citep{huang2018generalized}. The implementation is available at \url{https://github.com/Biwei-Huang/Generalized-Score-Functions-for-Causal-Discovery/}. We used cross-validation parameters \texttt{parameters.kfold = 10} and \texttt{parameters.lambda = 0.01}. 
\item Equal variance (EqVar) algorithm identifies source node by minimizing conditional variance in linear SEM \citep{chen2018causal}. The implementation is available at \url{https://github.com/WY-Chen/EqVarDAG}. The original EqVar algorithm estimates the error variances in a linear SEM via the covariance matrix $\Sigma=\E XX^{T}$, and then uses linear regression (e.g. best subset selection) to learn the structure of the DAG. We adapted this algorithm to the nonlinear setting in the obivous way by using GAMs (instead of subset selection) for variable selection. The use of the covariance matrix to estimate the order remains the same.
\item PC algorithm and greedy equivalence search (GES) are standard baselines for structure learning \citep{spirtes1991,chickering2003}. The implementation is available from R package \texttt{pcalg}. For PC algorithm, use correlation matrix as sufficient statistic. Independence test is implemented by \texttt{gaussCItest} with significance level \texttt{alpha=0.01}. For GES, set the score to be \texttt{GaussL0penObsScore} with \texttt{lambda} as $\log n/2$, which corresponds to the BIC score.
\end{itemize}

The experiments were conducted on an Intel E5-2680v4 2.4GHz CPU with 64 GB memory.

\subsection{Metrics}
\label{app:exp:metrics}

We evaluated the performance of each algorithm with the following two metrics:
\begin{itemize}
\item $\mathbb{P}(\text{correct order})$: The percentage of runs in which the algorithm gives a correct topological ordering, over $N$ runs. This metric is only sensible for algorithms that first estimate an ordering or return an adjacency matrix which does not contain undirected edges, including \RESIT, \CAM, \EqVar{} and \NOTEARS.
\item Structural Hamming distance (SHD): A standard benchmark in the structure learning literature that counts the total number of edge additions, deletions, and reversals needed to convert the estimated graph into the true graph.
\end{itemize}
Since there may be multiple topological orderings of a DAG, in our evaluations of order recovery, we check whether or not the order returned is any of the possible valid orderings. For \PC{}, \GES{}, and \GSGES{}, they all return a CPDAG that may contain undirected edges, in which case we evaluate them favourably by assuming correct orientation for undirected edges whenever possible. Since \CAM{}, \RESIT{}, \EqVar{} and \NPVAR{} each first estimate a topological ordering then estimate a DAG. To estimate a DAG from an ordering, we apply the same pruning step to each algorithm for a fair comparison, which is adapted from \citep{buhlmann2014}. Specifically, given an estimated ordering, we run a \texttt{gam} regression for each node on its ancestors, then determine the parents of the node by the $p$-values with significance level \texttt{cutoff=0.001} for estimating the DAG.

\subsection{Timing}
\label{app:exp:timing}

For completeness, runtime comparisons are reported in Tables~\ref{tab:timing:d20} and~\ref{tab:timing:d40}. Algorithms based on linear models such as \EqVar{}, \PC{}, and \GES{} are by far the fastest. These algorithms are also the most highly optimized. The slowest algorithms are \GSGES{} and \RESIT{}. Timing comparisons against \CAM{} are are difficult to interpret since by default, \CAM{} first performs preliminary neighbourhood search, which can easily be applied to any of the other algorithms tested. The dramatic difference with and without this pre-processing step by comparing Tables~\ref{tab:timing:d20} and~\ref{tab:timing:d40}: For $d=40$, with this extra step \CAM{} takes just over 90 seconds, whereas without it, on the same data, it takes over 8.5 hours. For comparison, \NPVAR{} takes around two minutes (i.e. without pre-processing or neighbourhood search).

\begin{table}[t]
\begin{center}
\begin{tabular}{llll}
\toprule
Algorithm & $d$ & $n$ & Runtime (s) \\
\midrule
\EqVar{} & 20 & 1000 & $0.0017 \pm 0.0003$ \\
\PC{} & 20 & 1000 & $0.056 \pm 0.016$ \\
\GES{} & 20 & 1000 & $0.060 \pm 0.034$ \\
\NPVAR{} & 20 & 1000 & $10.76 \pm 0.23$ \\
\NOTEARS{} & 20 & 1000 & $31.46 \pm 8.79$ \\
\CAM{} (w/ PNS) & 20 & 1000 & $40.56 \pm 1.29 $ \\
\CAM{} (w/o PNS) & 20 & 1000 & $559.01 \pm 9.49$ \\
\RESIT{} & 20 & 1000 & $652.15 \pm 7.26$ \\
\GSGES{} & 20 & 1000 & $3216.00 \pm 95.0$ \\
\bottomrule\\
\end{tabular}
\end{center}
\caption{Runtime comparisons for $d=20$. Timing for \CAM{} is presented with and without preliminary neighbourhood selection (PNS), which is a pre-processing step that can be applied to any algorithm. In our experiments, only \CAM{} used PNS.}
\label{tab:timing:d20}
\end{table}

\begin{table}[t]
\begin{center}
\begin{tabular}{llll}
\toprule
Algorithm & $d$ & $n$ & Runtime (s) \\
\midrule
\EqVar{} & 40 & 1000 & $0.0043 \pm 0.0003$ \\
\GES{} & 40 & 1000 & $0.12 \pm 0.0052$ \\
\PC{} & 40 & 1000 & $0.019 \pm 0.030$ \\
\NOTEARS{} & 40 & 1000 & $76.05 \pm 19.16$ \\
\CAM{} (w/ PNS) & 40 & 1000 & $95.59 \pm 6.33$ \\
\NPVAR{} & 40 & 1000 & $118.33 \pm 2.25$ \\
\CAM{} (w/o PNS) & 40 & 1000 & $31644.56 \pm 1329.31$ \\
\bottomrule\\
\end{tabular}
\end{center}
\caption{Runtime comparisons for $d=40$. Timing for \CAM{} is presented with and without preliminary neighbourhood selection (PNS), which is a pre-processing step that can be applied to any algorithm. In our experiments, only \CAM{} used PNS.}
\label{tab:timing:d40}
\end{table}

\subsection{Additional experiments}
\label{app:exp:expts}

Here we collect the results of our additional experiments. Since the settings MC-AGP-$k$ and MC-NGP-$k$ are equivalent (i.e. since there is only parent for each node), the plots for MC-NGP-$k$ are omitted. Some algorithms might be skipped due to high computational cost or numerical issue.
\begin{itemize}
\item Figure~\ref{fig:shd_vs_n_d5}: SHD vs. $n$ with $d=5$ fixed, across all graphs and models tested. 
\item Figure~\ref{fig:shd_vs_n_d10}: SHD vs. $n$ with $d=10$ fixed, across all graphs and models tested. 
\item Figure~\ref{fig:shd_vs_n_d20}: SHD vs. $n$ with $d=20$ fixed, across all graphs and models tested. 
\item Figure~\ref{fig:shd_vs_n_d40}: SHD vs. $n$ with $d=40$ fixed, across all graphs and models tested with \GSGES{} and \RESIT{} skipped (due to high computational cost). 
\item Figure~\ref{fig:shd_vs_n_d50}: SHD vs. $n$ with $d=50$ fixed, across all graphs and models tested with \GSGES{}, \RESIT{} and \NOTEARS{} skipped (due to high computational cost).
\item Figure~\ref{fig:shd_vs_n_d60}: SHD vs. $n$ with $d=60$ fixed, across all graphs and models tested with \GSGES{}, \RESIT{} and \NOTEARS{} skipped (due to high computational cost).
\item Figure~\ref{fig:shd_vs_n_d70}: SHD vs. $n$ with $d=70$ fixed, across all graphs and models tested with \GSGES{}, \RESIT{} and \NOTEARS{} skipped (due to high computational cost).
\item Figure~\ref{fig:glm}: SHD vs. $n$ with $d$ ranging from $5$ to $70$ on GLM with \CAM{} and \RESIT{} skipped (due to numerical issues).
\item Figure~\ref{fig:shd_vs_d_n500}: SHD vs. $d$ with $n=500$ fixed, across all graphs and models tested.
\item Figure~\ref{fig:shd_vs_d_n1000}: SHD vs. $d$ with $n=1000$ fixed, across all graphs and models tested.
\item Figure~\ref{fig:order_vs_n_d5}: Ordering recovery vs. $n$ with $d=5$ fixed, across all graphs and models tested. 
\item Figure~\ref{fig:order_vs_n_d10}: Ordering recovery vs. $n$ with $d=10$ fixed, across all graphs and models tested. 
\item Figure~\ref{fig:order_vs_n_d20}: Ordering recovery vs. $n$ with $d=20$ fixed, across all graphs and models tested.
\item Figure~\ref{fig:order_vs_n_d40}: Ordering recovery vs. $n$ with $d=40$ fixed, across all graphs and models tested with \RESIT{} skipped (due to high computational cost).
\item Figure~\ref{fig:order_vs_n_d50}: Ordering recovery vs. $n$ with $d=50$ fixed, across all graphs and models tested with \RESIT{} and \NOTEARS{} skipped (due to high computational cost).
\item Figure~\ref{fig:order_vs_n_d60}: Ordering recovery vs. $n$ with $d=60$ fixed, across all graphs and models tested with \RESIT{} and \NOTEARS{} skipped (due to high computational cost).
\item Figure~\ref{fig:order_vs_n_d70}: Ordering recovery vs. $n$ with $d=70$ fixed, across all graphs and models tested with \RESIT{} and \NOTEARS{} skipped (due to high computational cost).
\end{itemize}

\begin{figure}[h]
\centering
\includegraphics[width=\textwidth]{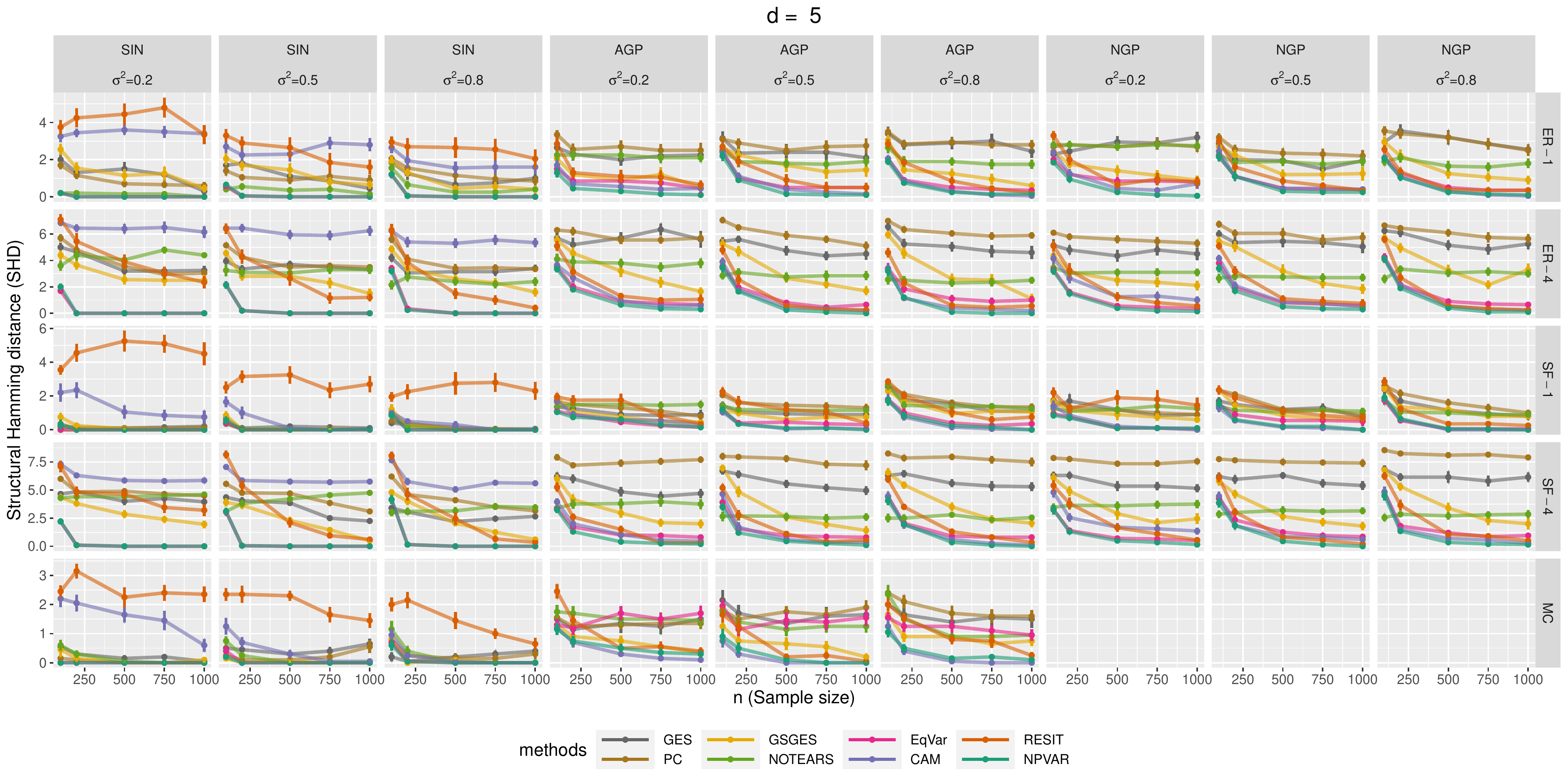}\\
\caption{SHD vs $n$ for fixed $d=5$.}
\label{fig:shd_vs_n_d5}
\end{figure}

\begin{figure}[h]
\centering
\includegraphics[width=\textwidth]{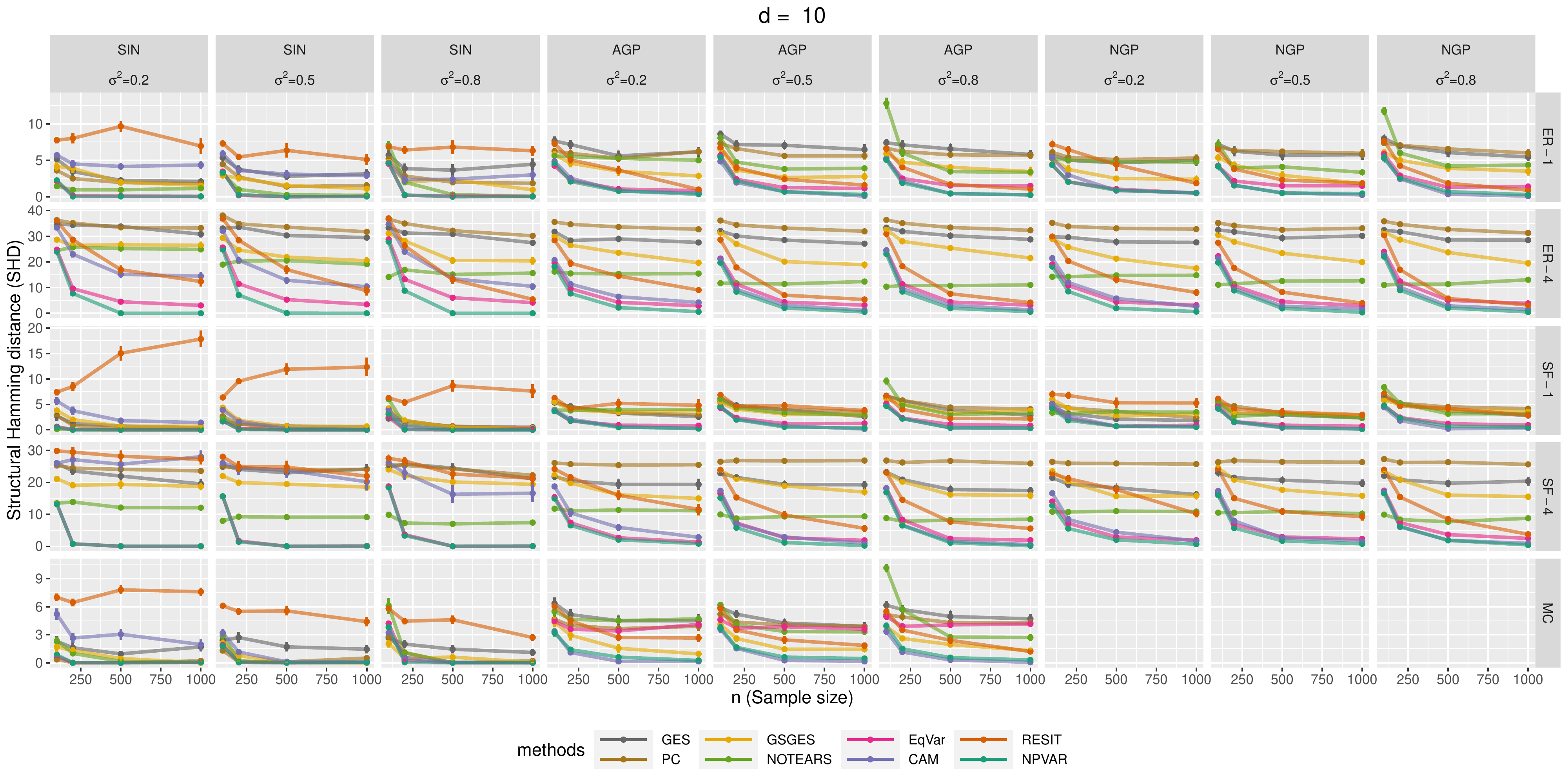}\\
\caption{SHD vs $n$ for fixed $d=10$.}
\label{fig:shd_vs_n_d10}
\end{figure}

\begin{figure}[h]
\centering
\includegraphics[width=\textwidth]{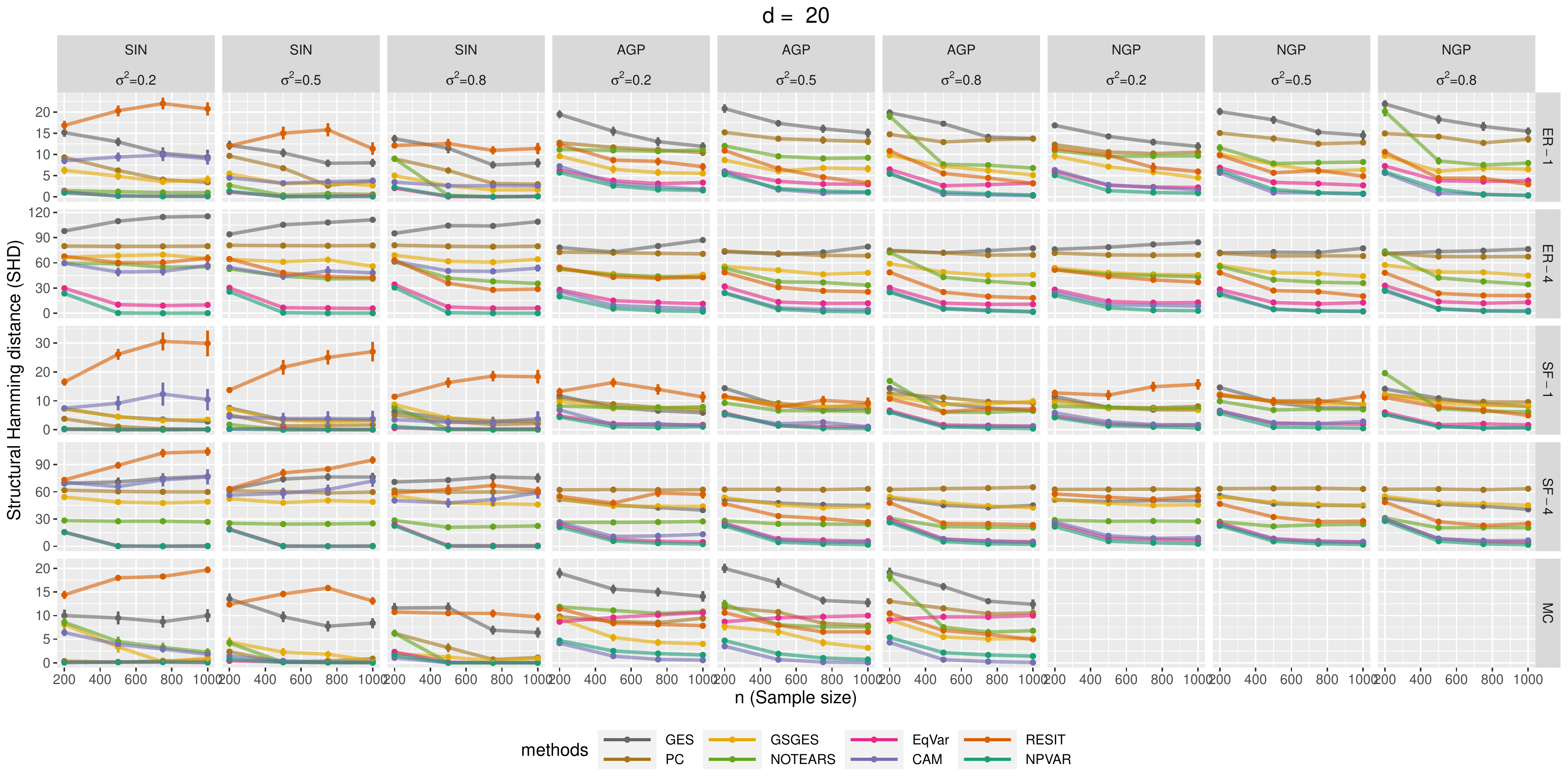}\\
\caption{SHD vs $n$ for fixed $d=20$.}
\label{fig:shd_vs_n_d20}
\end{figure}

\begin{figure}[h]
\centering
\includegraphics[width=\textwidth]{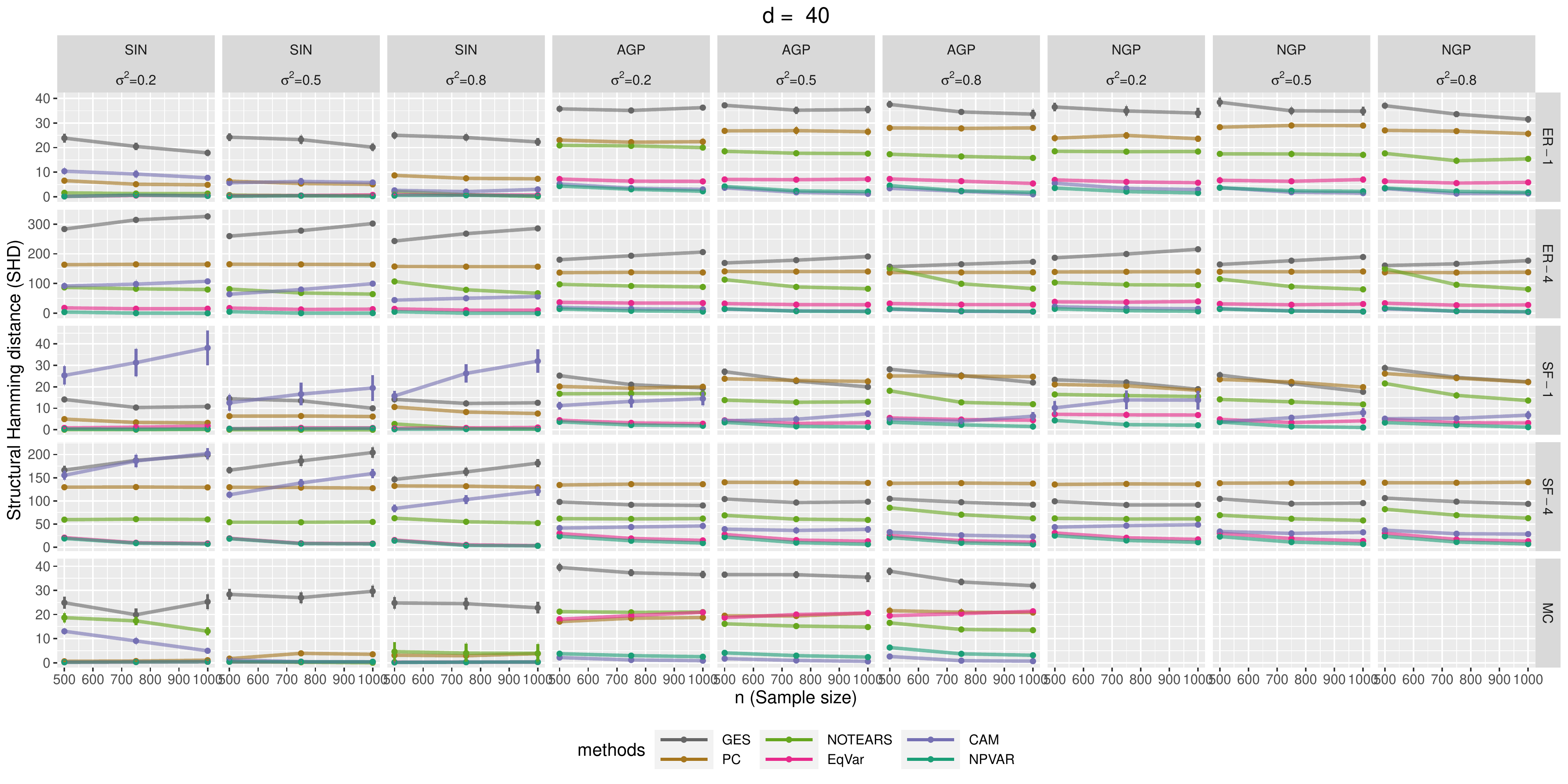}\\
\caption{SHD vs $n$ for fixed $d=40$.}
\label{fig:shd_vs_n_d40}
\end{figure}

\begin{figure}[h]
\centering
\includegraphics[width=\textwidth]{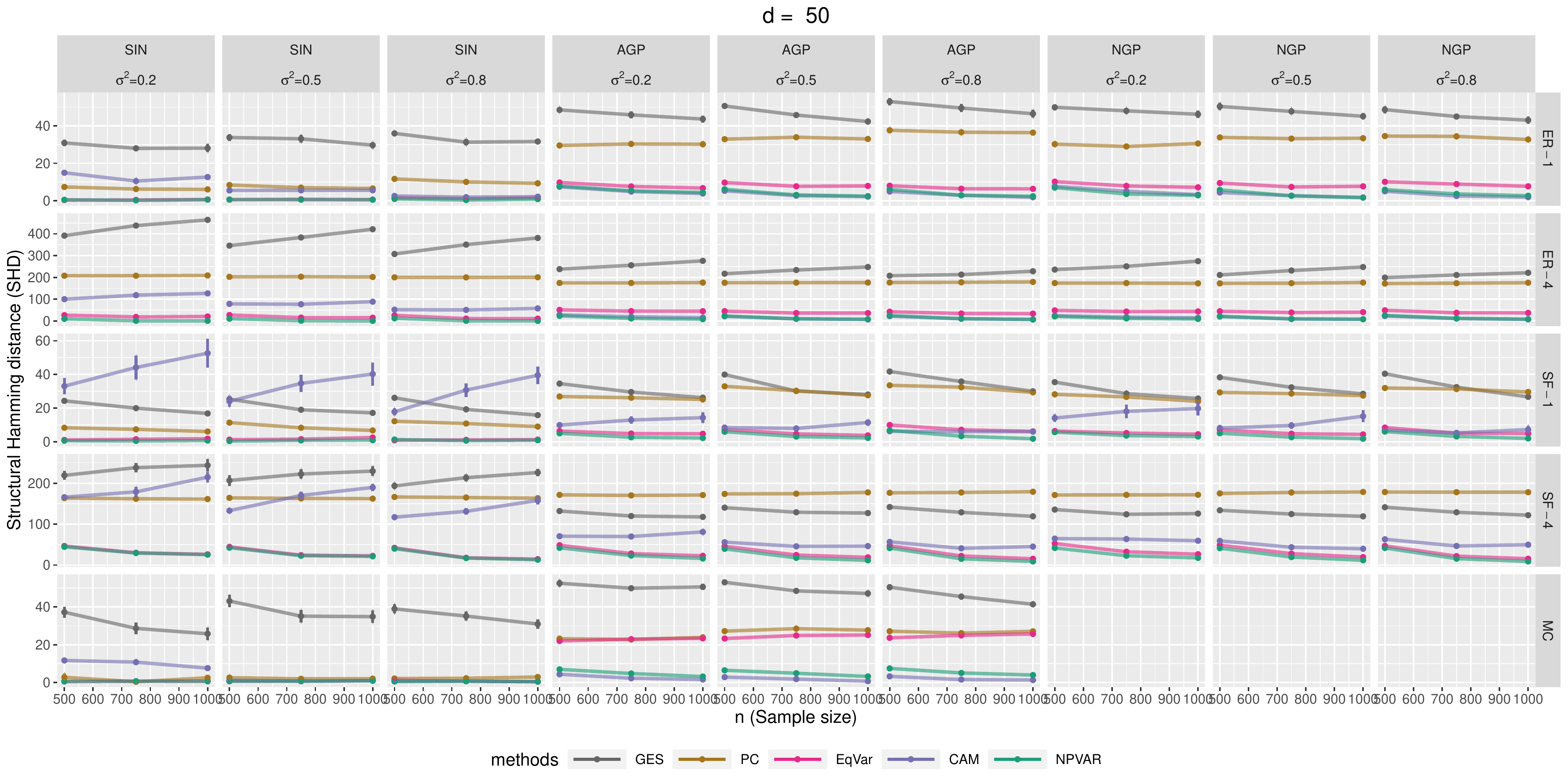}\\
\caption{SHD vs $n$ for fixed $d=50$.}
\label{fig:shd_vs_n_d50}
\end{figure}

\begin{figure}[h]
\centering
\includegraphics[width=\textwidth]{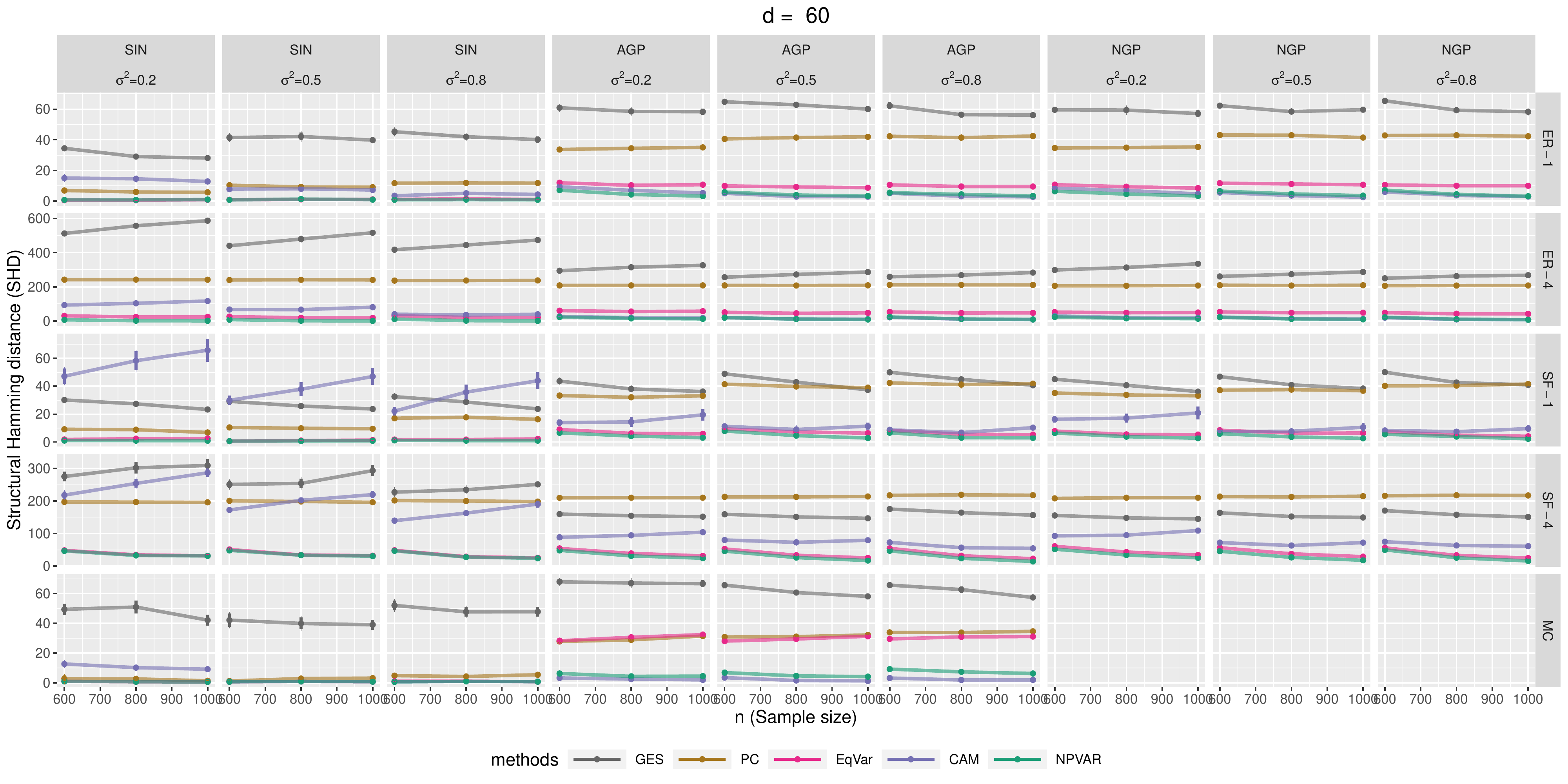}\\
\caption{SHD vs $n$ for fixed $d=60$.}
\label{fig:shd_vs_n_d60}
\end{figure}

\begin{figure}[h]
\centering
\includegraphics[width=\textwidth]{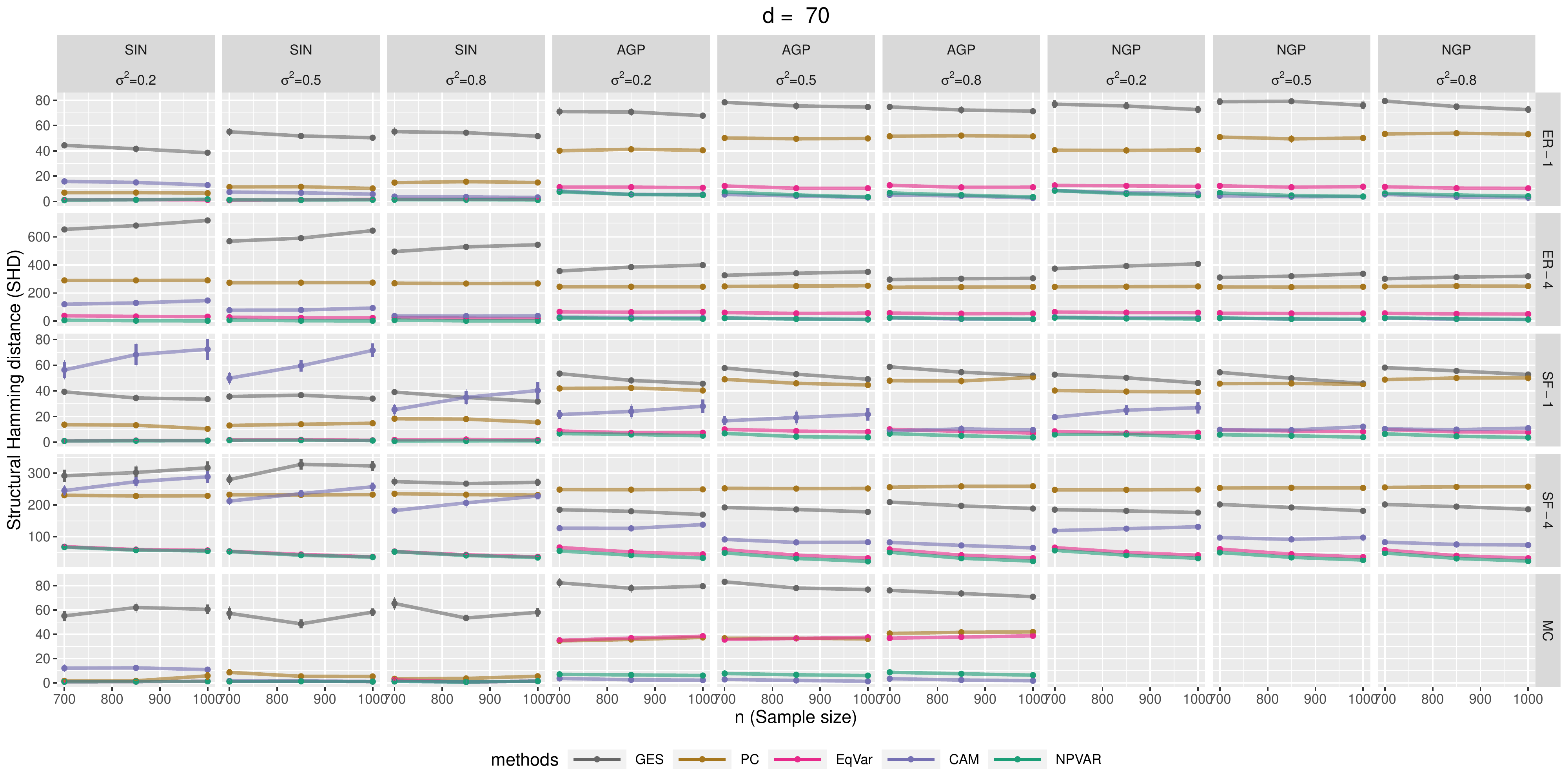}\\
\caption{SHD vs $n$ for fixed $d=70$.}
\label{fig:shd_vs_n_d70}
\end{figure}

\begin{figure}[h]
\centering
\includegraphics[width=\textwidth]{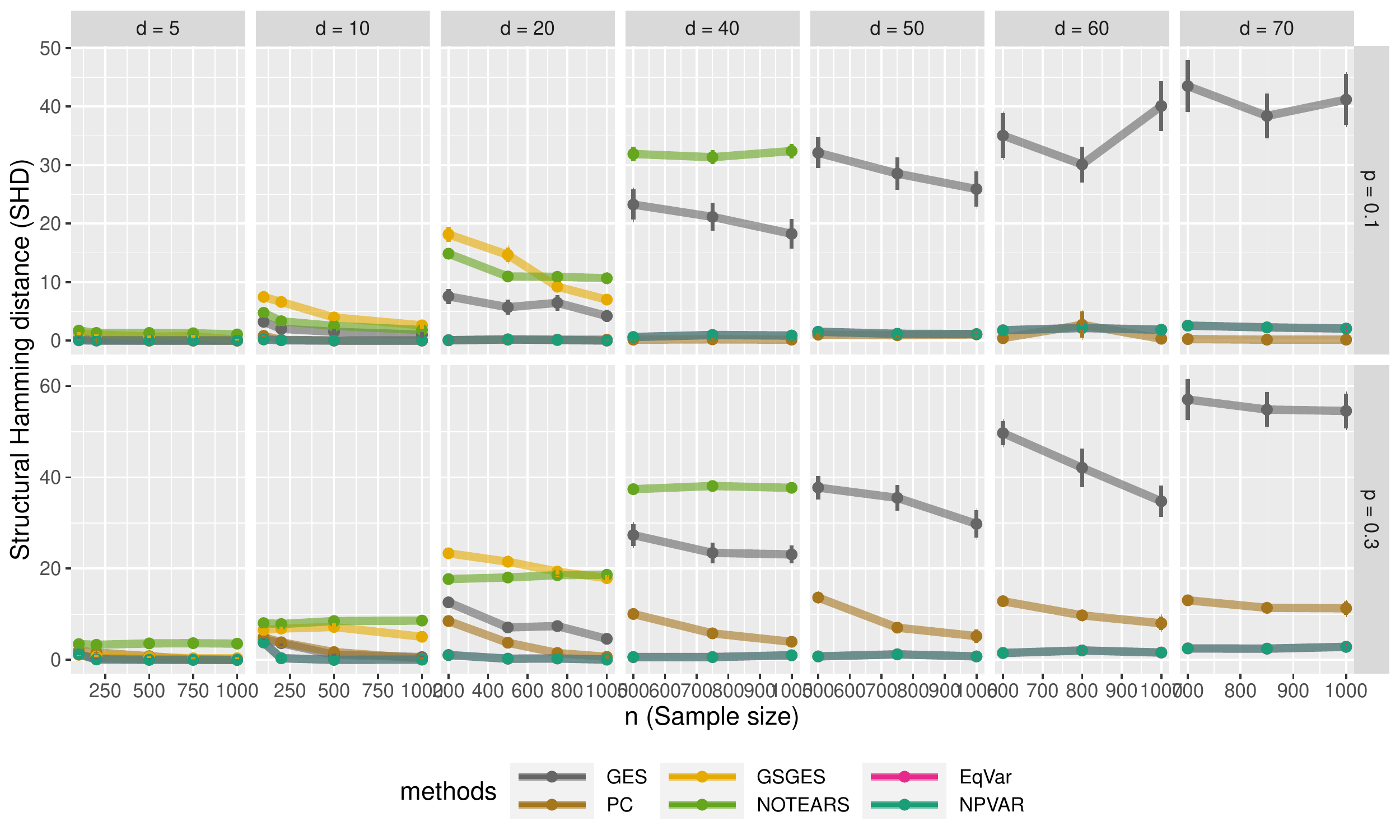}\\
\caption{SHD vs $n$ for different $d$ ranging from 5 to 70 on GLM}
\label{fig:glm}
\end{figure}

\begin{figure}[h]
\centering
\includegraphics[width=\textwidth]{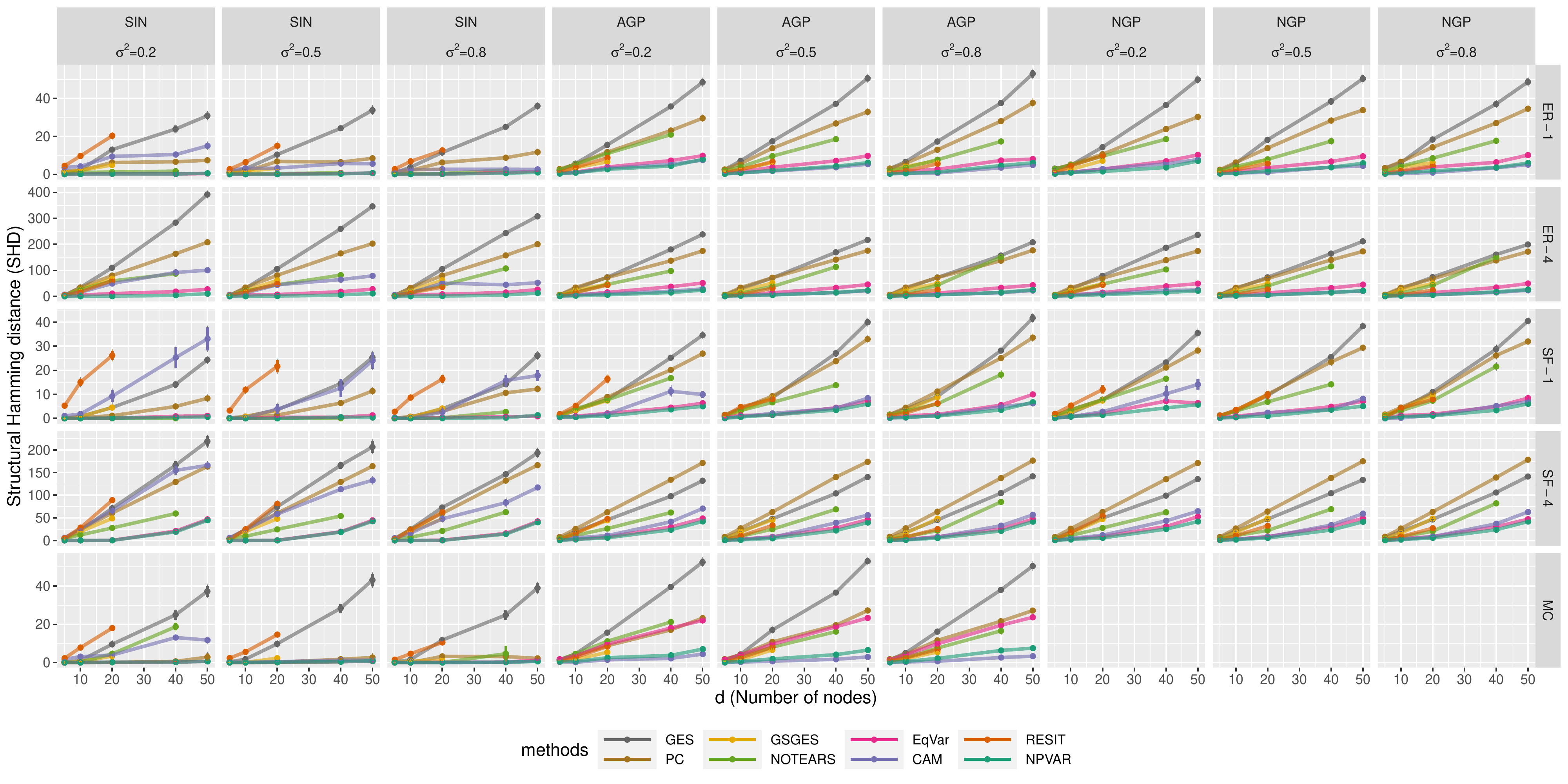}\\
\caption{SHD vs $d$ for fixed $n=500$.}
\label{fig:shd_vs_d_n500}
\end{figure}

\begin{figure}[h]
\centering
\includegraphics[width=\textwidth]{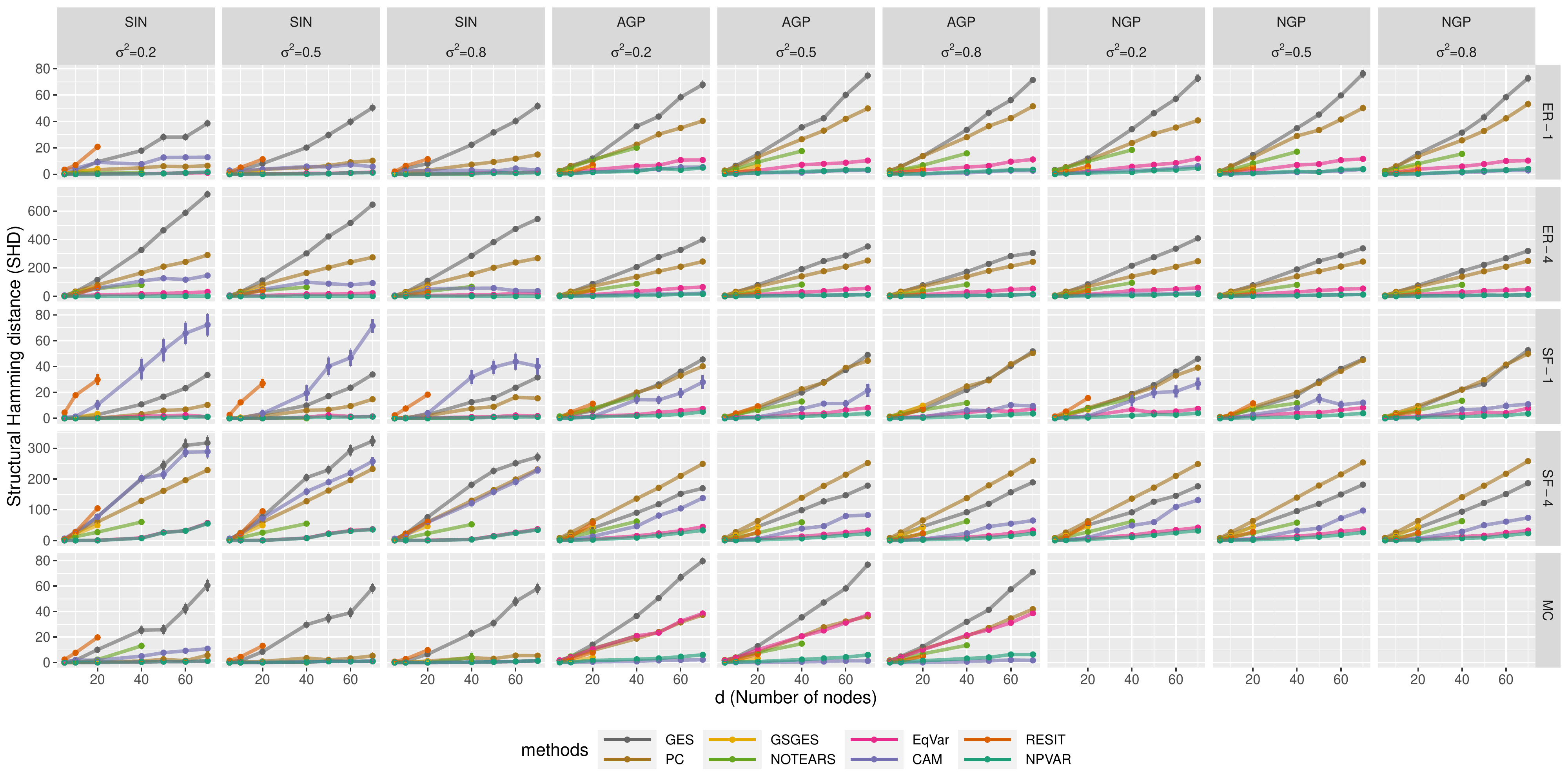}\\
\caption{SHD vs $d$ for fixed $n=1000$.}
\label{fig:shd_vs_d_n1000}
\end{figure}

\begin{figure}[h]
\centering
\includegraphics[width=\textwidth]{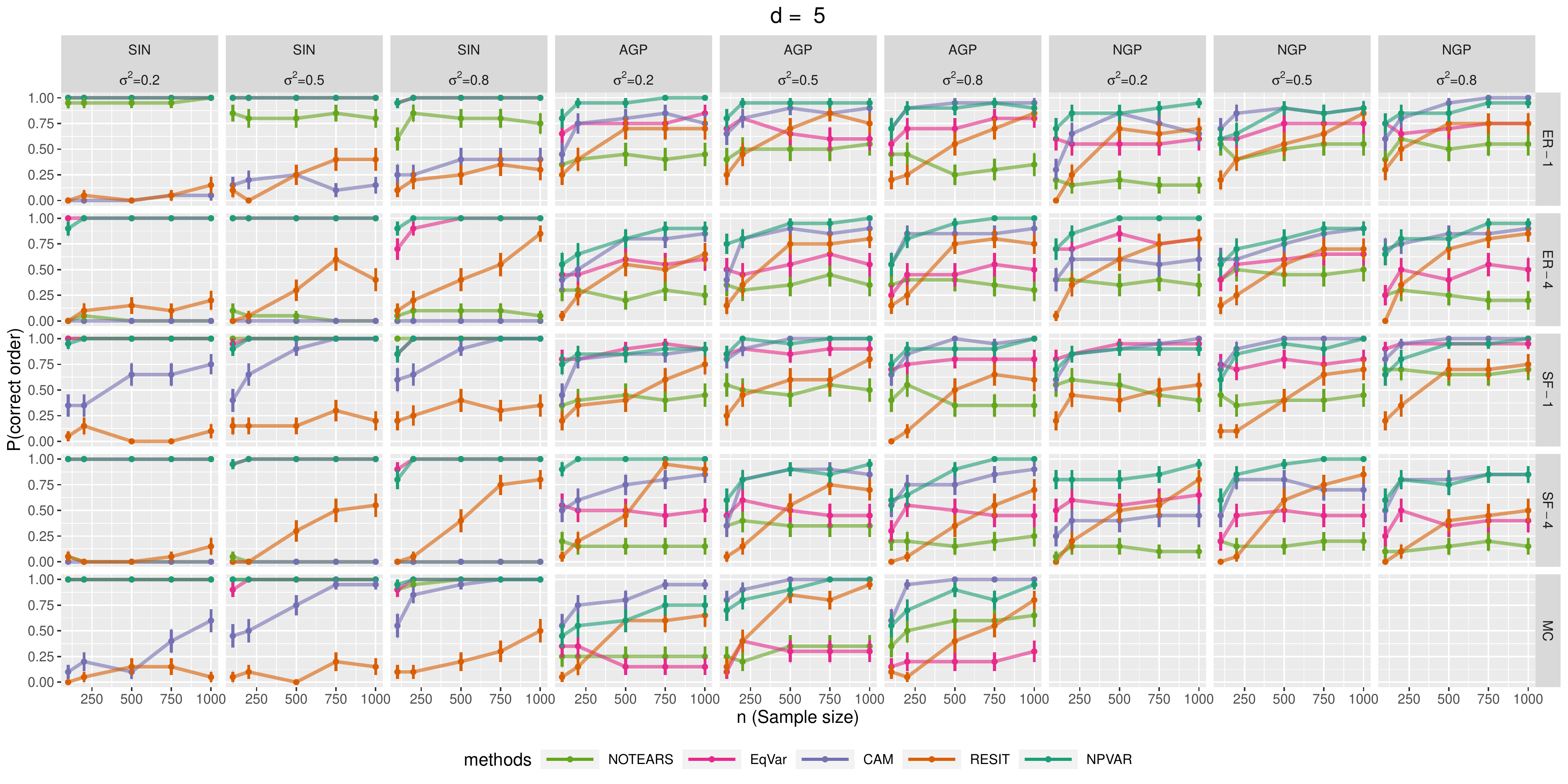}\\
\caption{Ordering recovery vs $n$ for fixed $d=5$.}
\label{fig:order_vs_n_d5}
\end{figure}

\begin{figure}[h]
\centering
\includegraphics[width=\textwidth]{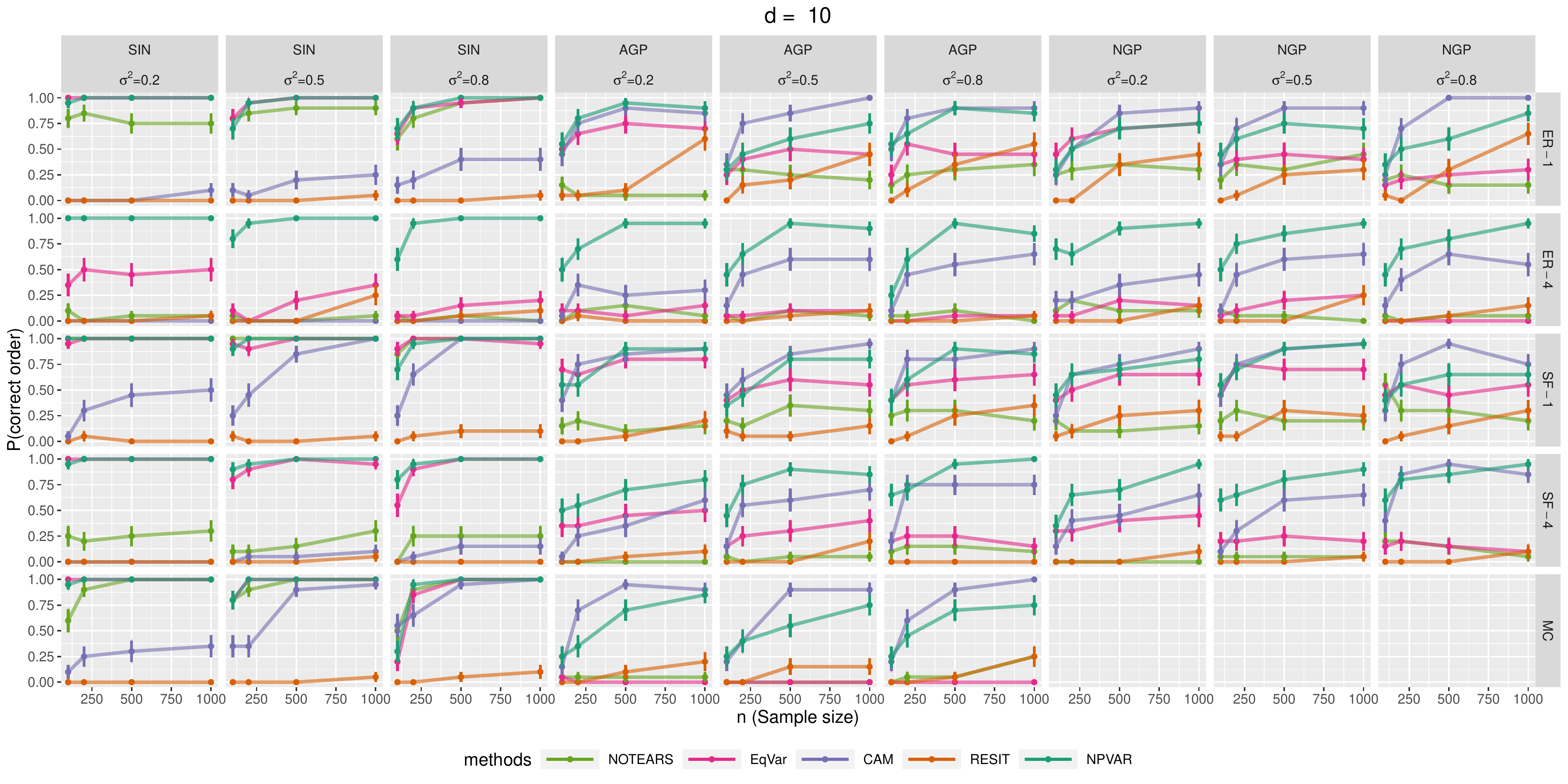}\\
\caption{Ordering recovery vs $n$ for fixed $d=10$.}
\label{fig:order_vs_n_d10}
\end{figure}

\begin{figure}[h]
\centering
\includegraphics[width=\textwidth]{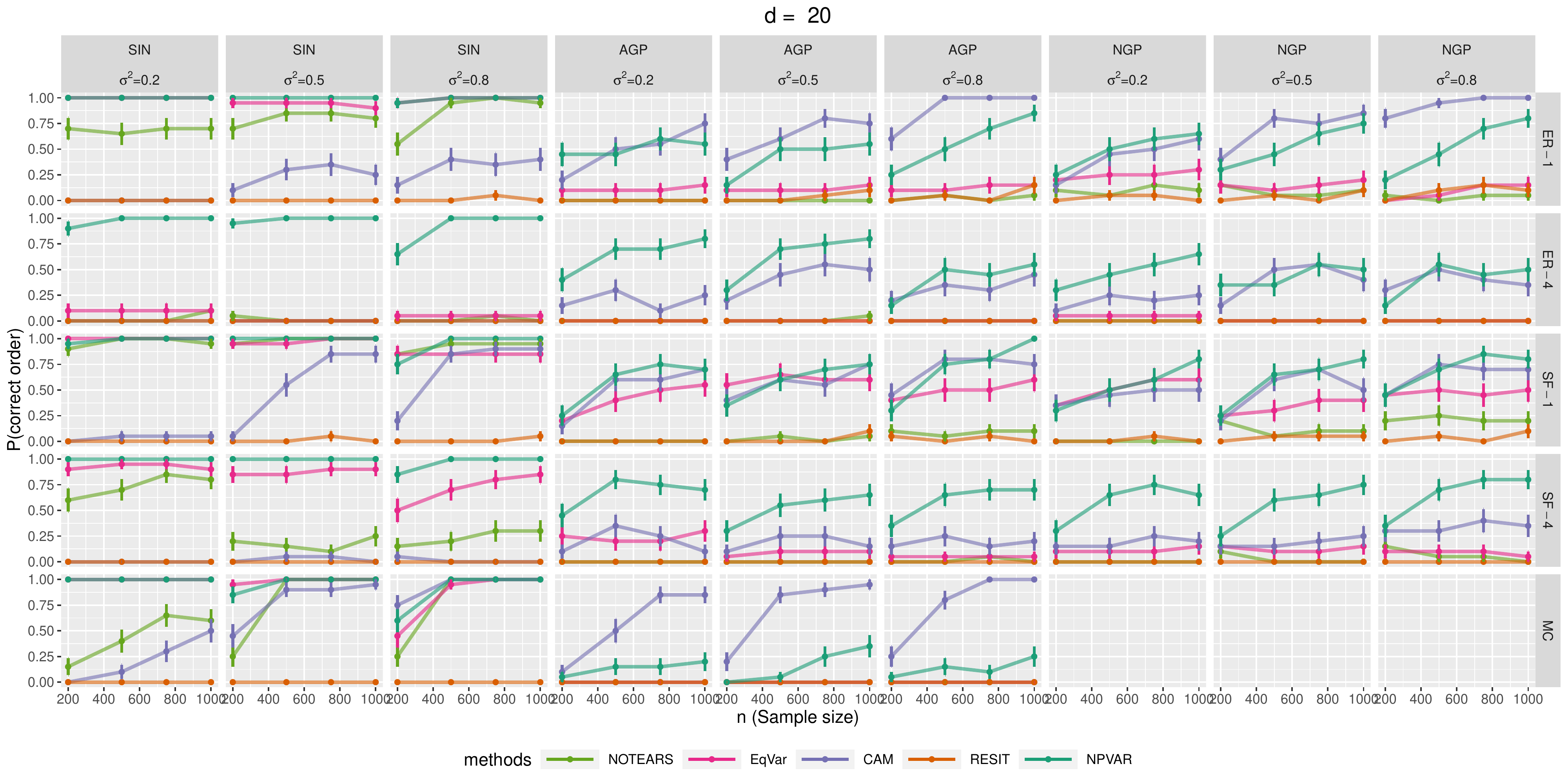}\\
\caption{Ordering recovery vs $n$ for fixed $d=20$.}
\label{fig:order_vs_n_d20}
\end{figure}

\begin{figure}[h]
\centering
\includegraphics[width=\textwidth]{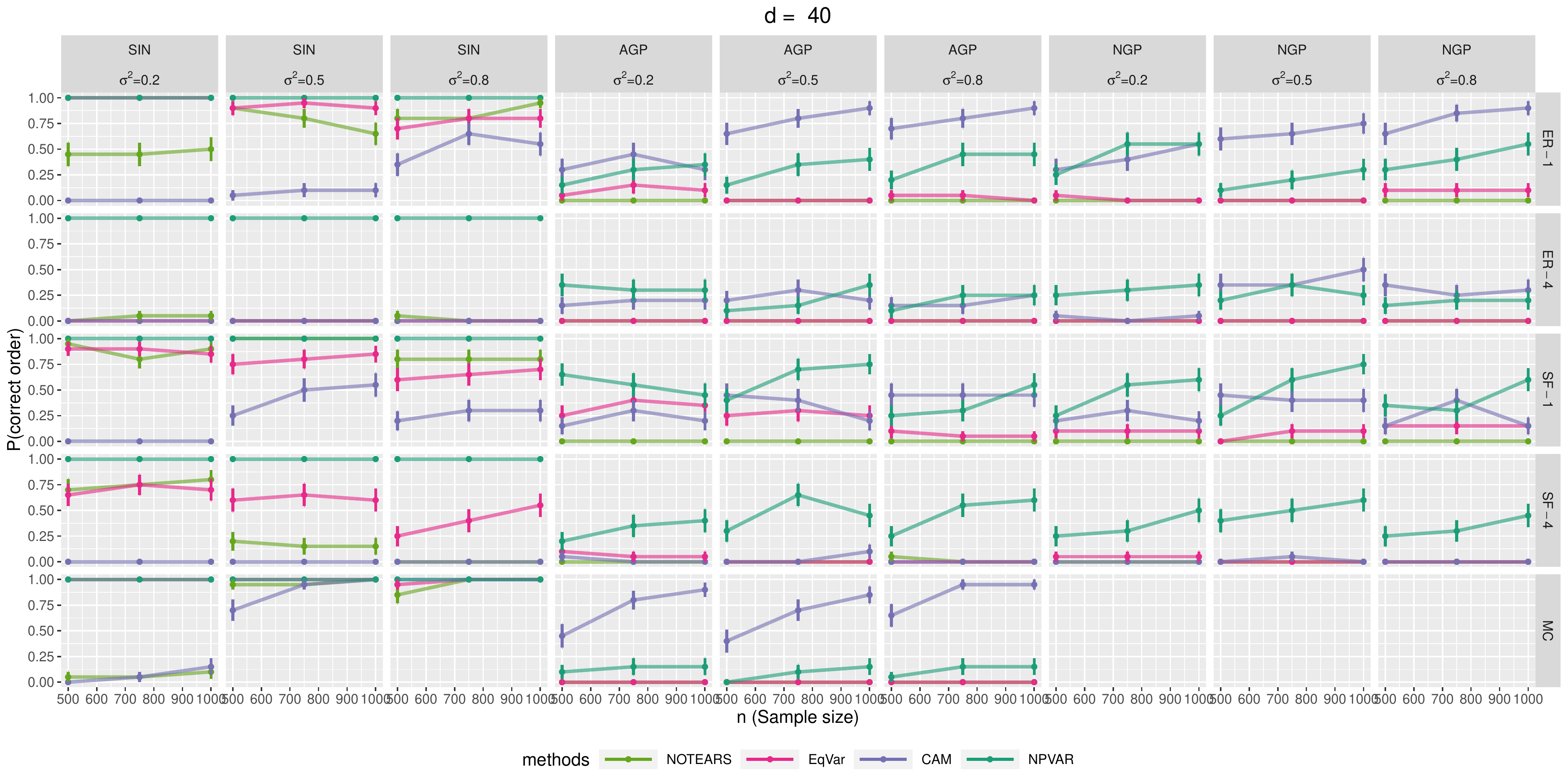}\\
\caption{Ordering recovery vs $n$ for fixed $d=40$.}
\label{fig:order_vs_n_d40}
\end{figure}

\begin{figure}[h] 
\centering
\includegraphics[width=\textwidth]{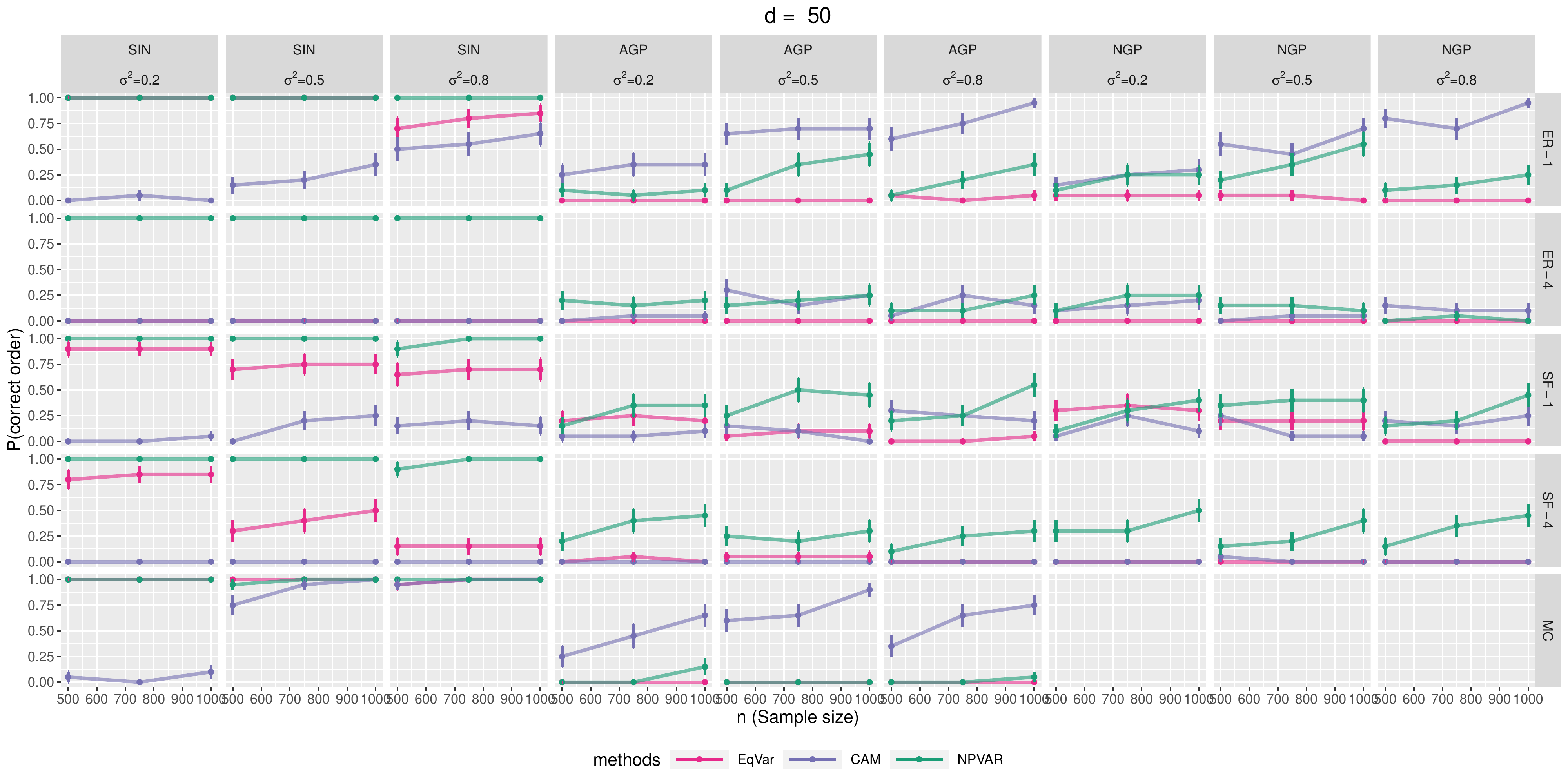}\\
\caption{Ordering recovery vs $n$ for fixed $d=50$.}
\label{fig:order_vs_n_d50}
\end{figure}

\begin{figure}[h]
\centering
\includegraphics[width=\textwidth]{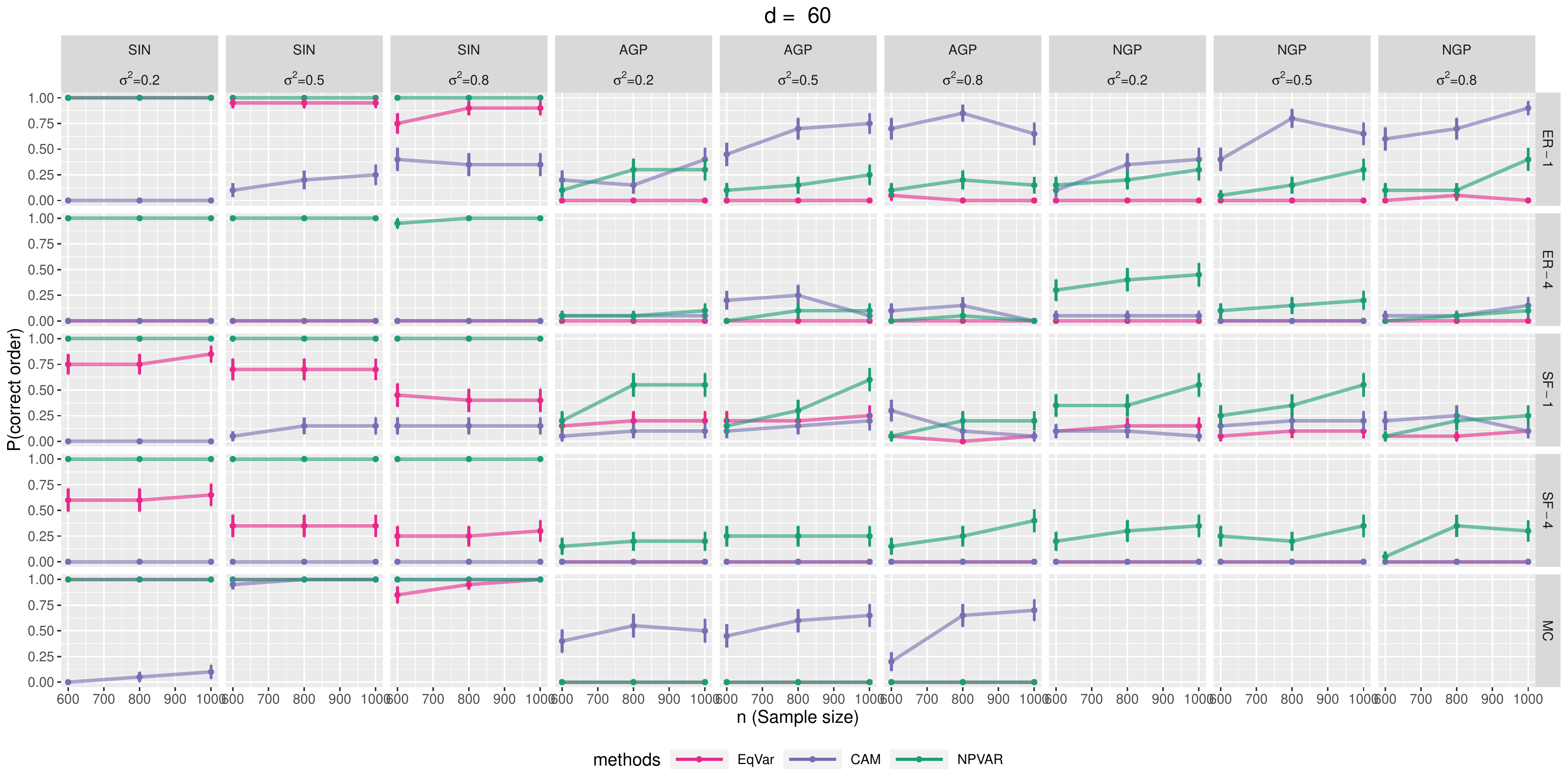}\\
\caption{Ordering recovery vs $n$ for fixed $d=60$.}
\label{fig:order_vs_n_d60}
\end{figure}

\begin{figure}[h]
\centering
\includegraphics[width=\textwidth]{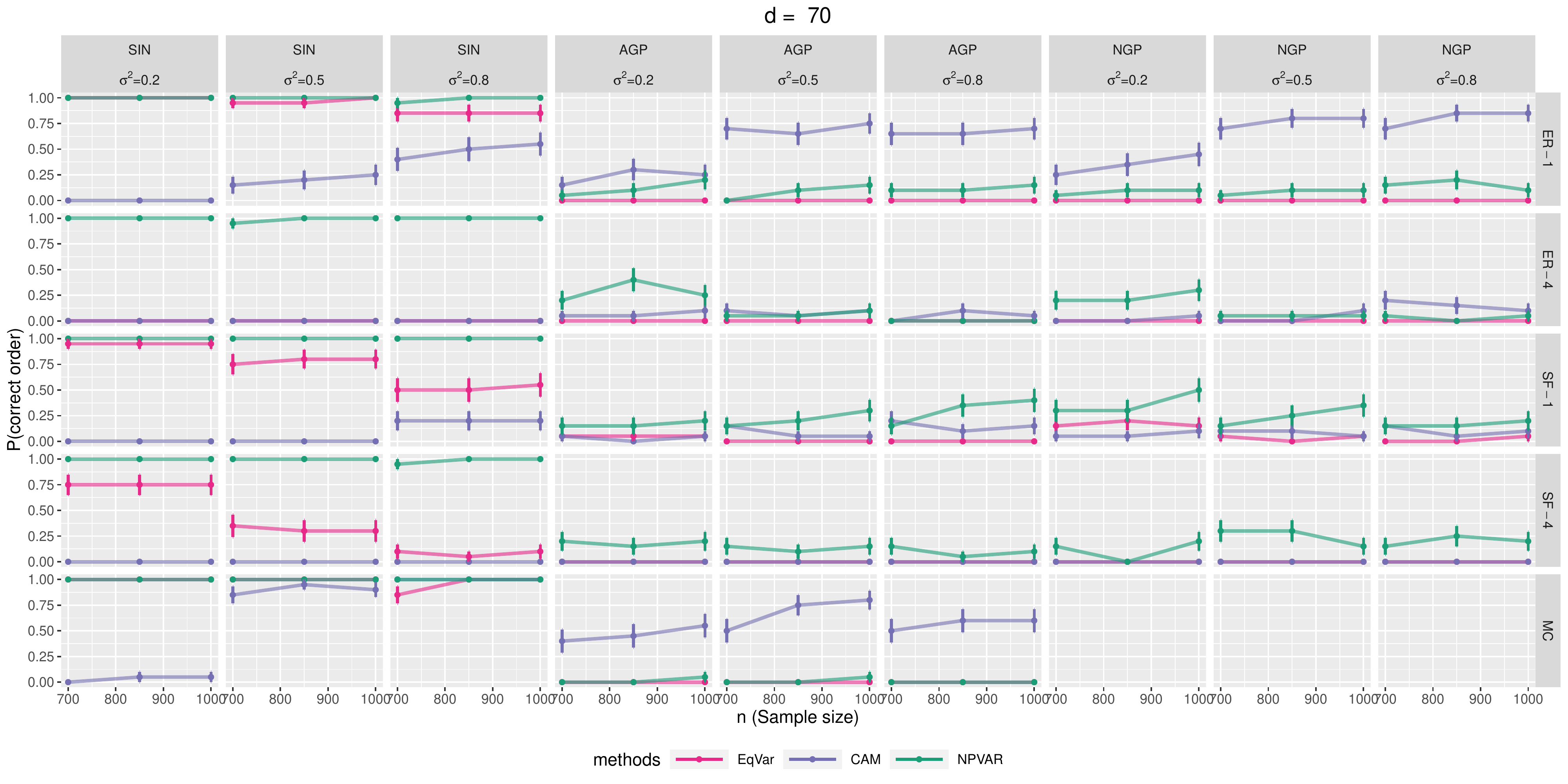}\\
\caption{Ordering recovery vs $n$ for fixed $d=70$.}
\label{fig:order_vs_n_d70}
\end{figure}

\clearpage
\bibliography{genbib,personalbib}
\bibliographystyle{bbib}
\end{document}